\pdfoutput=1
\documentclass{article}

\usepackage{graphicx}
\usepackage{amsmath}
\usepackage{amssymb}
\usepackage{booktabs}
\usepackage[pagebackref=true,breaklinks,colorlinks,citecolor=blue,linkcolor=blue,urlcolor=blue,hypertexnames=false]{hyperref}
\usepackage{url}
\usepackage{xcolor}
\usepackage{amsthm}
\usepackage{thmtools}
\usepackage{thm-restate}
\makeatletter
\renewcommand\thmt@autorefsetup{%
  \@xa\def\csname\thmt@envname autorefname\@xa\endcsname\@xa{\thmt@thmname}%
}
\makeatother
\usepackage{wrapfig}

\usepackage{enumitem}
\usepackage{subcaption}

\newtheorem{theorem}{Theorem}

\newtheorem{remark}[theorem]{Remark}
\theoremstyle{definition}
\newtheorem{definition}[theorem]{Definition}

\newcommand{\leo}[1]{#1}

\DeclareMathOperator{\SE}{SE}
\DeclareMathOperator{\se}{\mathfrak{se}}
\DeclareMathOperator{\so}{\mathfrak{so}}
\DeclareMathOperator{\diam}{diam}

\DeclareMathOperator{\rank}{rank}
\DeclareMathOperator{\id}{id}
\newcommand{\R}{\mathbb{R}}

\newcommand{\tightpara}[1]{\paragraph{#1}}

\usepackage[most]{tcolorbox} 

\newtcolorbox{taskbox}[1]{
    colback=gray!5,
    colframe=black,
    fonttitle=\bfseries,
    title=#1,
    arc=0mm, 
    boxrule=0.5pt
}

\usepackage[capitalize]{cleveref}
\PassOptionsToPackage{numbers}{natbib}

\usepackage[main, final]{neurips_2026}

\usepackage[utf8]{inputenc} 
\usepackage[T1]{fontenc}    
\usepackage{hyperref}       
\usepackage{url}            
\usepackage{booktabs}       
\usepackage{amsfonts}       
\usepackage{nicefrac}       
\usepackage{microtype}      
\usepackage{xcolor}         

\title{SeeSE3: Emergence of 3D Space in Vision Features}

\author{%
  Caroline Chen \\
  Google DeepMind \\
  \And
  Sayna Ebrahimi \\
  Google DeepMind \\
  \And
  Fedor Kitashov \\
  Google DeepMind \\
  \AND
  Ming-Hsuan Yang \\
  Google DeepMind \\
  \And
  Leonidas Guibas \\
  Google DeepMind \\
  \And
  Viorica P\u atr\u aucean \\
  Google DeepMind \\
  \And
  Maks Ovsjanikov \\
  Google DeepMind \\
}

\begin{document}

\maketitle

\begin{abstract}
In this paper, we ask whether vision foundation models construct representations that reflect the structure of rigid motions in 3D Euclidean space. Unlike previous works that probe 3D awareness of vision features by regressing image-centric quantities such as depth or normals, we investigate the relation between the \textit{structure} of the space of visual features  and the group of Euclidean transformations $SE(3)$. We propose a set of probes to evaluate this relation from both topological and geometric perspectives: a mutual neighborhood metric that measures the alignment between feature neighborhoods and spatial \textit{topology}, and a Poincaré Adapter to test the linear accessibility of the \textit{geometry} of camera motion from latent displacements in static scenes. We show that self-supervised vision models, which, in principle, have not been trained with direct 3D supervision or active agency, possess latent subspaces in which feature differences are strongly correlated with relative camera motion in $SE(3)$, when probed with a lightweight adapter. Building on this insight we propose ``Latent-Space Navigation'', a proof of concept approach that predicts, by vector arithmetic in the latent space, the features of the view reached after a given camera motion, bypassing explicit 3D reconstruction. Our implementation is available at \url{https://github.com/google-deepmind/seese3}.
\end{abstract}


\section{Introduction}
\label{sec:intro}

Learning to perceive the 3D structure of the world is integral to human development \cite{tolman1948cognitive,gibson2014ecological,shepard1971mental,piaget2013child,klatzky1998allocentric}. In contrast, in most existing vision, robotics or world models, the geometry of 3D space is imposed \textit{a priori}, either through an explicit choice of a coordinate system, e.g., \cite{schonberger2016structure,mildenhall2021nerf,savva2019habitat,zhang2025efficiently} or by distilling from human knowledge. We therefore ask whether a vision model can ``discover'' the geometry of rigid motion in the 3D world with minimal intervention, e.g., purely from visual statistics.

A prominent view in both cognitive science and philosophy is that to discover the geometry of Euclidean space, one must have \textit{agency}, e.g., \cite{berkeley1709essay,held1963movement,piaget2013child}. For example, Henri Poincaré hypothesized that ``A motionless being could never have acquired the concept of space because, he would have had no reason to distinguish [changes of position] from changes of state. Nor would he have been able to acquire it if his movements had not been \textit{voluntary}.'' (\cite{poincare1905science} Chap. 4).

Modern vision foundation models present an opportunity to test this hypothesis. These models are, in a strict sense, ``motionless beings.'' Even when trained on video, they are \textit{passive} observers of someone else's movement. Inspired by Poincaré, we thus formulate the following problem:

\paragraph{The Poincaré Task:}
Can a motionless observer, given only passive visual input, ``discover'' the structure of 3D space, in Poincar\'e's sense of its group of rigid displacements? In particular, does a vision model naturally organize its latent space in a way that 
reflects the special Euclidean group $SE(3)$ of Euclidean transformations? Following Poincar\'e~\cite{poincare1905science}, and Klein's Erlangen program~\cite{klein1872vergleichende}, we identify the structure of space with its group of rigid displacements; we do not ask whether features encode the 3D geometry of the scene itself, such as depth or a reconstruction.


\leo{Note that the latent spaces of vision models associate  features with the \textit{content} of the scene, which is naturally present in 3D, while 3D \textit{space itself} is not directly observable.} Unlike prior works that either probe the 3D awareness of visual features by regressing \leo{local} \leo{visible-surface-centric} quantities \leo{such as depth and normals} \cite{elbanani2024probing}, or employ high-capacity decoders to extract pose \cite{teed2023deep,wang2024dust3r,wang2025vggt}, we test whether \textit{the feature space itself} is organized to reflect the group structure of rigid motions, $SE(3)$.

%
%
The key characteristics of the group of rigid motions are its \textit{low-dimensionality} (6D for camera position and orientation), and its \textit{homogeneity}, e.g., relative displacements have the same meaning regardless of the position in space. This is seemingly in contrast to vision features which are both high-dimensional and fundamentally tied to visual
\leo{content}. It is thus unclear whether such a low-dimensional, homogeneous coordinate system exists in the features of even the most powerful models.

\begin{figure}
    \centering
    \includegraphics[width=0.96\linewidth]{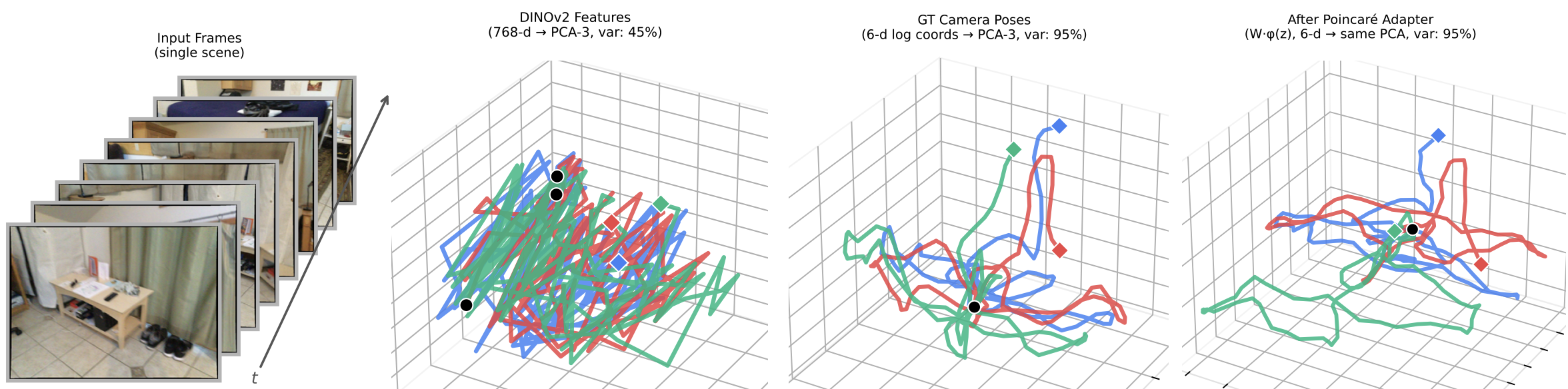}
    \caption{\textit{Overview:} we investigate whether vision features are aligned with the geometry of Euclidean transformations, as captured by motion (trajectories) within a static scene (left). While raw features are chaotic and tangled (middle left), we show that a lightweight network that we call the ``Poincaré adapter'' can identify the subspace of the feature space, associated with spatial motion, and unroll it to  closely follow the ground truth camera trajectories (middle right and right).}
    \label{fig:teaser}
\end{figure}

To answer this question, we define an appropriate  protocol and test several vision foundation models on static scenes taken from diverse datasets. Our experiments show that information contained in several purely self-supervised image models, when probed with a lightweight adapter, is strongly aligned with the structure of $SE(3)$ --- so these models are able to ``see'' $SE(3)$. 
We build on this insight to propose efficient ways to regress camera motion and to enable closed-form latent visual navigation.
The key contributions of this work are:
\begin{enumerate}
    \item We formulate the problem of testing whether the structure of visual features reflects the group of Euclidean transformations $SE(3)$, and introduce several probes to evaluate the performance of a range of encoders with respect to this task.
    \item We show that while \textit{no} features are directly isomorphic to $SE(3)$ even locally, this structure can nevertheless be extracted using a lightweight Siamese decoder that we call the ``Poincaré adapter,'' which produces a homogeneous coordinate system.
    \item We provide a theoretical analysis of the conditions under which self-supervised methods can recover geometric structure, and investigate the \textit{difficulty} of recovering different components of $SE(3)$ --- for example, we show that \textit{rotation} is easier to encode than translation. We also study which training choices matter for the emergence of geometric structure, through an ablation of DINOv2 models trained on the same data.
    \item Building on this insight, we introduce a novel closed-form visual navigation approach, which exploits the linearization enabled by our Poincaré adapter.    
\end{enumerate}

\section{Related Work}
\label{sec:related}

\textbf{Probing 3D Awareness.} 
Our work is closely related to the recent literature on probing the 3D awareness of vision foundation models. For example, Probe3D \cite{elbanani2024probing} proposed a protocol and evaluated a range of vision models on depth and surface normals, as well as multi-view feature consistency (keypoint matching). Similarly, You et al. \cite{you2024multiview} demonstrated that multi-view feature consistency is strongly correlated with improved performance on various downstream tasks, and FiT3D~\cite{yue2024improving} improved 3D awareness through 3D-aware fine-tuning, while Amir et al. \cite{amir2021deep} showed that self-supervised
ViT features serve as effective dense visual descriptors for keypoint
correspondence. Huang et al. \cite{huang2025much} evaluated a range of vision models by estimating multiple 3D properties from their features via shallow read-outs, while Man et al.~\cite{man2024lexicon3d} performed a similar analysis in the context of scenes, by considering both low-level geometric as well as semantic 3D tasks. Other works have probed for viewer-centric properties such as depth (e.g., \cite{bhattad2023stylegan,chen2023beyond} among many others). Our work is fundamentally
different from all these efforts, in that we focus on the underlying \textit{structure} of the vision features themselves, and specifically on the presence of a homogeneous \textit{spatial coordinate system} within them, which, we argue, is the hallmark of intrinsic geometric discovery.


\textbf{Explicit 3D Reconstruction and Visual Odometry.} 
The field of explicit 3D estimation has seen rapid progress with models like DUSt3R \cite{wang2024dust3r} and VGGT \cite{wang2025vggt}. These approaches treat 3D reconstruction as a sequence-to-sequence problem, employing heavy Transformer architectures to ``compute'' geometry from features. Similarly, classical Visual Odometry (VO) \cite{campos2021orb} explicitly solves for pose. However, these approaches typically employ high-capacity decoders, which might demonstrate that geometric information \textit{exists} within the features, but do not reveal whether visual features \textit{are organized} \leo{and structured} in a way that reflects Euclidean 3D space. 

Closely related to our goals is RUST \cite{sajjadi2023rust}, which demonstrated that a generative model trained on \textit{static scenes} naturally organizes its latent space into a geometric structure isomorphic to physical camera parameters. Our work broadens this by asking if discriminative models trained on \textit{arbitrary, 
\leo{diverse} data} (like ImageNet or YouTube) can \textit{also discover this geometric substructure} for static scenes. Mitchel et al. \cite{mitchel2025true} show that in Novel View Synthesis, letting the network ``discover'' pose space through transferability, rather than imposing it a priori leads to better performance.

\textbf{Cognitive and Neuroscientific Foundations.}
Our investigation is also motivated by the work of H{\'e}naff et al. \cite{henaff2021primary}, who demonstrated that the primary visual cortex (V1) transforms curved video trajectories into ``straighter'' paths---a hypothesis recently extended to temporal straightening and learned world models~\cite{parthasarathy2023self,niu2024learning,bagad2026chirality,maes2026leworldmodel}. We ask whether vision foundation models perform a similar linearization to recover the 6D \textit{generators of rigid motion}. Our finding that camera motion can be decoded from the features of passive, motionless observers with a simple adapter bears on the view that grid-cell-like representations require active motor agency for path integration~\cite{sorscher2023unified,dorrell2026if}.

\textbf{Visual Navigation Models and Navigation World Models.} 
Our work is also related to Visual Navigation Models (VNMs) \cite{shah2023gnm,shah2023vint,sridhar2024nomad,doshi2024scaling,ren2025prior,guerrier2026can} and Navigation World Models \cite{bar2025navigation,yao2025navmorph,mur2026v,maes2026leworldmodel}. The former typically take a source and target frame and output a set of actions to reach the target from the source without building a map, while the latter (NWMs) are typically tasked with predicting the next visual state or observation given a current state and a simulated action. Although our work does not directly deal with \textit{action} selection, these two problems are closely related to our pose linearization and latent space navigation (Section~\ref{sec:protocol} below). Furthermore, both VNMs and NWMs have so far been driven primarily by architectural and scaling advances. Instead, our  work aims to establish and improve the fundamental mechanisms that enable navigation and prediction in vision feature space.

\section{Background and Motivation}
\label{sec:background}
The overall objective of our work is to study the structure of vision features and to relate it to the group of  Euclidean transformations. To make the problem well-posed, we consider \textit{static scenes}, where the mapping from camera pose to an image (and thus feature) is stable and well-defined.

\tightpara{Problem Statement.}
\label{sec:problem_formulation}

Let $\mathcal{S}$ be a static scene observed by a moving camera across time $t$. An encoder $f_\theta$ maps the image $I_t$ at time $t$ to a high-dimensional feature vector $z_t = f_\theta(I_t)$. Our goal is to determine if these features reflect the structure of Euclidean transformations the camera undergoes. 

In particular, we consider image acquisition as a sampling process $\mathcal{U} \rightarrow \mathbb{R}^d$ from the camera pose space $\mathcal{U} \subseteq SE(3)$ to the feature space $\mathbb{R}^d$. For a  static scene, this process is defined via composition $P_t \rightarrow I_t \rightarrow z_t$, where $P_t \in \mathcal{U}$ is the camera-to-world pose at time $t$, $I_t$ is the image associated with that pose, and $z_t$ is the vision feature. After assembling the set of features  $\mathcal{Z} = \{z_t\}_{t=1..N}$ we then compare the properties of the point set  $\mathcal{Z}$ to the structure of $SE(3)$, as defined by the set $\{P_t\}$.

\tightpara{Topology Consistency and Submanifold Dimensionality.}
We first evaluate whether the \textit{topology} of the feature point set $\{z_t\}$ locally respects the topology of $SE(3)$ by comparing the nearest neighbors in the feature space to nearest neighbors in camera pose space (metric \textbf{M1} below). We then evaluate whether $\{z_t\}$ locally forms a 6D manifold (metric \textbf{M2} below).

\tightpara{Geometry -- Local Linearization and Lie Algebra.}
Most importantly, we are interested in evaluating whether $\{z_t\}$ (or a projection of it) captures the \textit{geometry} of $SE(3)$, or, more precisely, whether it is \textit{isomorphic} to $SE(3)$. Unfortunately, it is generally impossible to construct a global isomorphic representation of $SE(3)$ with additive operations in $\mathbb{R}^d$, regardless of the dimensionality~$d$ (see e.g.,~\cite{murray1994mathematical,zhou2019continuity}). We thus focus on recovering the structure of $SE(3)$ \textit{locally}.

Formally, we let $P_t, P_{t+s} \in SE(3)$ be the camera-to-world pose matrices at frames $t$ and $t+s$. The relative transformation between them is given by $P_{rel} = P_{t+s}^{-1} P_t$. We map this relative pose to its tangent vector $\Delta P$ in the Lie Algebra $\mathfrak{se}(3)$ using the matrix logarithm:
\begin{equation}\label{eq:body_displacement}
    \Delta P = (\text{Log}(P_{t+s}^{-1} P_t))^{\vee} \in \mathbb{R}^6\,.
\end{equation}
Here, $\Delta P$ is a 6-dimensional vector comprising 3 translational velocities and 3 rotational velocities (axis-angle), and $\vee$ is the vee operator, which extracts the 6D twist coordinates from the $4{\times}4$ matrix in $\mathfrak{se}(3)$.\footnote{In all experiments, $\Delta P$ is computed as the translation of $P_{rel}$ together with the axis-angle vector of its rotation, which coincides with the twist coordinates for small motions.} We argue that if the latent space is geometrically structured, there should exist a fixed and uniform (i.e., applicable for every point in space) linear operator $W$ such that:
\begin{equation}
    \Delta P \approx W(z_{t+s} - z_t).
\end{equation}
The requirement for $W$ to be linear ensures that the latent space \textit{itself} supports vector arithmetic isomorphic to the tangent space of $SE(3)$.

 \tightpara{Space Discovery.} To link this optimization problem to the original Poincaré task,  suppose $\Delta P = W (z_{t+s} - z_t)$. Since $W \in \mathbb{R}^{6 \times d}$, this implies that transitions in the feature space $z$ can be decomposed into a 6D row space of $W$ (``changes of position''). Crucially, the modifications within this subspace are homogeneous: i.e., the difference vectors $z_{t+s}-z_{t}$, after projection, can be captured by \textit{the same} set of 6 basis vectors, regardless of $t$. This means that, \textit{by isolating such a subspace} an agent can, in principle, discover the primary degrees of freedom in the 3D physical world, and those degrees of freedom have a consistent meaning irrespective of the visual content observed.

\section{Protocol}
\label{sec:protocol}

We evaluate a range of vision encoders on static scenes from ScanNet \cite{dai2017scannet}, ARKitScenes \cite{baruch2021arkitscenes}, TUM RGB-D \cite{sturm2012benchmark}, 12 Scenes \cite{valentin2016learning}, and 7 Scenes \cite{shotton2013scene} datasets. As mentioned above, we introduce several metrics of increasing complexity to probe spatial awareness of vision models.

\tightpara{Metric 1 (M1): Topological Alignment.} We use the mutual $k$-nn alignment metric \cite{huh2024platonic,zhu2026dynamic} to evaluate the similarity between neighborhoods in the visual feature and camera pose spaces, respectively. While previous works have used it to compare signals as captured by different encoders (vision vs. text), here we adapt this metric to compare vision encoders \textit{directly} to spatial transformations as captured by camera motion in a static scene. We describe our precise protocol in Appendix Sec.~\ref{appendix:mutual_knn}.

\tightpara{Metric 2 (M2): Intrinsic Dimensionality (ID).}
We estimate the Local Intrinsic Dimensionality using MLE \cite{levina2004maximum} and Two-NN \cite{facco2017estimating} estimators. An intrinsic dimensionality closer to 6 suggests a representation that potentially captures the underlying degrees of freedom.

Note that both metrics \textbf{M1} and \textbf{M2} are \textit{training-free} and directly evaluate the structural similarity between the feature space of different encoders and camera pose space. While the results depend on the choice of representation for camera pose, we found that the relative ordering of different vision encoders is remarkably robust across different representations.

\tightpara{Metric 3 (M3): Linear Equivariance.}
We test if the feature space is \textit{globally} Euclidean. For this, we train a linear regressor $W$ to predict the relative pose change $\Delta P$ from the feature difference: $\Delta P \approx W (z_{t+s} - z_t)$ across frames separated by stride $s$. Success here implies the manifold is natively flat. The regression targets are standardized (zero mean, unit variance) per component to account for scale differences between translational and rotational units (see Appendix~\ref{app:implementation_details} for details).

\begin{figure}[t!]
    \centering
    \includegraphics[width=0.85\linewidth]{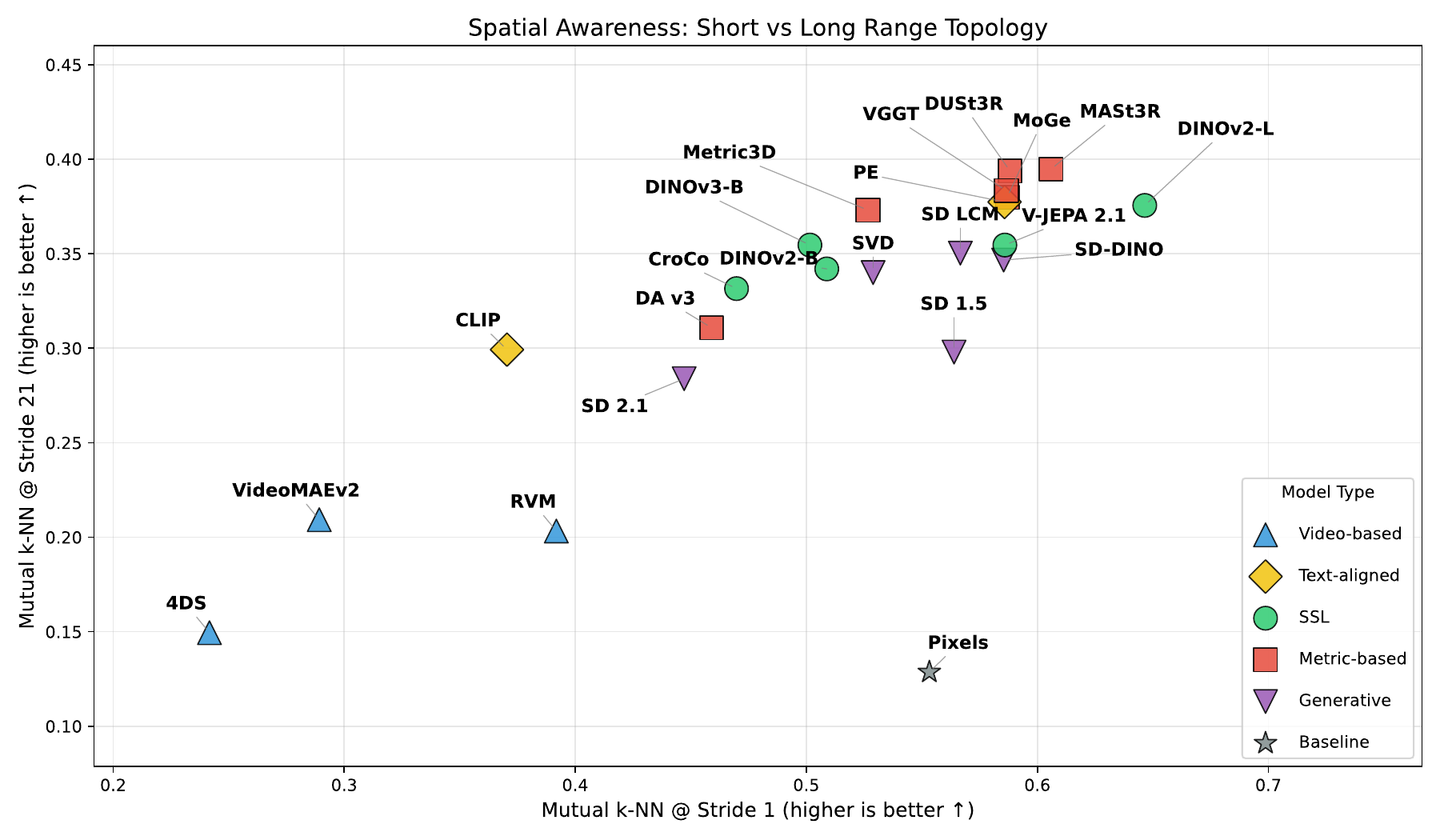}
    \caption{\textbf{Mutual $k$-nn Alignment (M1).} Alignment between camera poses and vision features for short ($x$ axis) and long ($y$ axis) strides on ScanNet. The former captures pixel-level dependencies, while the latter is more indicative of spatial awareness. Explicit geometric models (DUSt3R, MASt3R, MoGe) achieve the highest alignment, together with VGGT, whose features we take from its DINOv2 backbone. Among self-supervised models, DINO-family encoders show emergent spatial topology, while video models (VideoMAEv2, RVM, 4DS) struggle to maintain a consistent spatial embedding.}
    \label{fig:stride_sweep}
\end{figure}

\begin{table}[t!]
\centering
\small
\resizebox{\linewidth}{!}{
\begin{tabular}{lcccccccc}
\toprule
\textbf{Metric} & \textbf{Pixels} & \textbf{CroCo}~\cite{weinzaepfel2022croco} & \textbf{DINOv2-B}~\cite{oquab2023dinov2} & \textbf{DINOv3-B}~\cite{simeoni2025dinov3} & \textbf{DUSt3R}~\cite{wang2024dust3r} & \textbf{V-JEPA}~\cite{mur2026v} & \textbf{VGGT-L}~\cite{wang2025vggt} \\
\midrule
ID (TwoNN) & 5.01 & 7.66 & 9.15 & 10.76 & 7.16 & 9.55 & 7.74 \\
ID (MLE) & 4.03 & 5.92 & 6.41 & 7.88 & 4.40 & 5.13 & 5.36 \\
$R^2$ (\textbf{M3}) & -5.29 & -0.42 & -0.28 & -0.26 & -0.22 & 0.09 & -0.09 \\
\bottomrule
\end{tabular}
}
\caption{\textbf{Intrinsic Dimension \& Linear Equivariance.} Representative models compared on intrinsic dimension (TwoNN and MLE) and raw linear readout $R^2$ (\textbf{M3}). Values correspond to the layer maximizing \textbf{M3} performance (e.g., layer~9 for DINOv2-B; our DINOv2-B adapter and navigation experiments use layer~7, where the adapter performs best, Table~\ref{tab:ablation}). We provide the full table with 18 different models in Appendix~\ref{app:full_intrinsic_dim}.}
\label{tab:intrinsic_dim}
\end{table}

\tightpara{Metric 4 (M4): The Poincaré Adapter.}
Since the metric \textbf{M3} above imposes a restrictive linear constraint, we introduce a lightweight, trainable adapter $\varphi_\theta$ in a Siamese manner:
\begin{equation}
    \Delta P \approx W (\varphi_\theta(z_{t+s}) - \varphi_\theta(z_t))
    \label{eq:pose_prediction}
\end{equation}
We call the network $\varphi_\theta$ a ``Poincaré Adapter,'' since its goal is to unroll the non-linear feature space to construct a homogeneous coordinate system, where changes of pose become linear.
We emphasize that the adapter network $\varphi$ is applied \textit{independently} to any latent vector $z$, thus allowing a \textit{restructuring} of the latent space to reveal the $SE(3)$ geometry. Please see Figure \ref{fig:teaser} for a qualitative illustration. All metrics use one global feature vector per image (the CLS token when available, mean-pooled patch tokens otherwise), so no patches are compared across images. $\varphi_\theta$ is a small MLP and $W$ a bias-free linear readout, trained with AdamW and an MSE loss on z-scored targets; within each scene, frames are split in time into 80\% training, 10\% validation (early stopping) and 10\% test. For \textbf{M3}, $W$ is fitted by Ridge regression ($\alpha = 1$) on the first 80\% of the frames. Per-experiment settings are listed in Appendix~\ref{app:implementation_details}.
We investigate alternatives to Eq.~\eqref{eq:pose_prediction} in Section~\ref{sec:analysis} and Appendix~\ref{app:alternatives}.


Following recent work (e.g., \cite{mason2026large}) showing that \textit{intermediate} layers carry more geometric awareness, for all metrics \textbf{M1-M4} we sweep the encoder blocks and report the best layer.

\subsection{Application: Latent Space Navigation}
\label{sec:navigation_protocol}

As a proof of concept enabled by our study, we apply our approach to \textbf{Latent Space Navigation}. The key idea is that, if the visual latent space can be structured to reflect the local topology of $SE(3)$, then an agent should be able to ``imagine'' the visual consequence of a physical movement without requiring explicit 3D reconstruction. This objective relates to both latent novel view synthesis \cite{mildenhall2021nerf,metzer2023latentnerf}, and, more closely, to Navigation World Models \cite{bar2025navigation,mur2026v}. Unlike both of these frameworks, our goal is not to synthesize views,  however, but rather to show that navigation can be greatly simplified (reduced to vector arithmetic) after identifying the appropriate camera-aligned \textit{linear feature subspace}.

Formally, given the adapted feature $g_t = \varphi_\theta(z_t)$ of a starting frame and a desired physical displacement (pose change) $\Delta P$, our goal is to predict the feature representation of the destination frame, $\hat{g}_{t+s}$ by inverting the Poincaré Adapter. We evaluate a completely training-free navigator (which we call Inverse Poincaré). Since our adapter projects features such that $\Delta P \approx W(g_{t+s} - g_t)$, (where $g_t = \varphi_{\theta}(z_t)$ as in Eq.~\eqref{eq:pose_prediction}), we can linearly invert this relationship:
\begin{equation}
    \hat{g}_{t+s} = g_t + W^\dagger \Delta P\,,
\end{equation}
where $W^\dagger$ is the pseudo-inverse of the linear projection weights. We compare this approach to parametric correctors (e.g., Linear, MLP, and Attention layers) trained to predict $\hat{g}_{t+s}$ from the concatenated input $[g_t; \Delta P]$. To give the networks a strong geometric foundation, the higher-capacity models (MLP and Attention) are constructed to predict a non-linear \textit{residual} correction on top of the baseline prediction (i.e., $\hat{g}_{t+s} = g_t + W^\dagger \Delta P + \text{Net}([g_t; \Delta P])$).


To evaluate navigation success, we perform a nearest-neighbor retrieval task. We retrieve the nearest neighbor of the predicted $\hat{g}_{t+s}$ among all frames of the scene (training and test frames) via $L_2$ distance in the adapted feature space, and measure \textbf{Pose-Error Hits} (Hit@$\epsilon$). For each frame, we compute the $L_2$ distance between the raw 6D twist vectors (3D translation in meters + 3D axis-angle rotation in radians) of the target and retrieved relative poses; $L_2$ on the axis-angle vector $\omega = \mathrm{Log}(R)^\vee$ approximates the SO(3) geodesic distance for small rotations, and Table~\ref{tab:navigation_results} also reports the geodesic error. As this distance mixes translation and rotation, the thresholds $\epsilon \in \{0.1, 0.2, 0.3, 0.5\}$ correspond approximately to a pure translation error of $\epsilon$\,m, a pure rotation error of $\epsilon$\,rad, or some combination thereof. Hit@$\epsilon$ is the fraction of retrievals with distance below $\epsilon$.

\section{Main Results}
\label{sec:results}

We first evaluate a diverse set of vision foundation models on static scenes from ScanNet~\cite{dai2017scannet}. Our evaluation suite includes: patchified Raw Pixels (baseline), CLIP~\cite{radford2021learning}, DINOv2~\cite{oquab2023dinov2}, DINOv3~\cite{simeoni2025dinov3}, CroCo~\cite{weinzaepfel2022croco} (3D-pretext SSL), V-JEPA 2.1~\citep{mur2026v}, against explicit geometric foundation models DUSt3R~\cite{wang2024dust3r} and MoGe~\cite{wang2025moge}, and further metric-depth, generative and video models (Fig.~\ref{fig:stride_sweep} and Appendix~\ref{app:full_intrinsic_dim}).


\tightpara{Topological Alignment (Metric M1).}
We first evaluate if the neighborhood structure of the visual latent space mirrors the neighborhood structure of the physical world. We use the mutual $k$-nn alignment score described in Appendix~\ref{appendix:mutual_knn}, across both short and long temporal strides. A higher stride introduces larger viewpoint changes, reducing pixel overlap and testing global spatial awareness.

As shown in Figure~\ref{fig:stride_sweep}, models explicitly trained with geometric objectives (DUSt3R, MoGe, MASt3R) achieve the highest alignment scores ($\approx 0.35$--$0.40$) at high strides. This highlights that their representations map visual inputs to a metric-aligned coordinate system. Interestingly, although most of these methods require \textit{a decoder} to regress camera pose, our results show that the \textit{encoder itself} tends to structure the latent space to align with the structure of the pose space.

Perhaps more surprisingly, among purely self-supervised vision models, DINOv2 and
DINOv3 exhibit strong topological alignment ($\approx 0.34$--$0.38$), clearly outperforming CLIP, the video models and Raw Pixels. Crucially, this structure emerges purely from passive image-level self-supervision, without explicit depth or pose signals. We also note that video models (VideoMAE~\cite{wang2023videomae}, RVM~\cite{zoran2025rvm}, 4DS~\cite{carreira2024scaling}) perform less successfully on this metric ($\approx 0.15$--$0.25$), possibly because their latent states entangle spatial layout with temporal motion (Appendix~\ref{app:model_axes}).

\tightpara{Geometry of Raw Features (Metrics M2 \& M3).}
We estimate the Local Intrinsic Dimensionality (ID) using both Two-NN and MLE estimators. Theoretically, a perfectly space-aligned representation of a static scene viewed by a moving camera should have an ID of $\approx 6$.
%
As shown in Table~\ref{tab:intrinsic_dim}, most encoders result in features with intrinsic dimensionality around 6 (we justify this effect theoretically in Section \ref{sec:analysis}), while CroCo, pre-trained with explicit 3D cross-view completion, shows a lower ID ($\approx 5.9$, MLE) closer to the physical degrees of freedom.

We then test if physical displacements $\Delta P$ can be recovered via a global linear projection $\Delta P \approx W(z_{t+s} - z_t)$. As reported in Table~\ref{tab:intrinsic_dim}, \textbf{models generally fail this test}, yielding negative or near-zero $R^2$ scores (e.g., DINOv2 $R^2 \approx -0.28$, Pixels $R^2\approx-5.29$). This confirms that the ``Manifold of Vision'' is natively curved; standard vector arithmetic on raw features does not correspond to physical motion. We provide a full comparison across a wide range of models in Appendix~\ref{app:full_intrinsic_dim}.

\tightpara{Emergence of Linear Structure (Metric M4).} As shown above, Metrics \textbf{M2} and \textbf{M3} do not tend to differentiate vision encoders because they are, respectively, too easy (most models have ID close to 6) and too hard (no model possesses an easily linearizable subspace aligned with motion). We thus focus on the Poincaré Adapter described in Eq.~\eqref{eq:pose_prediction}. Appendix Figure~\ref{fig:scannet_examples_main} shows example frame pairs at a stride of $s{=}40$ frames in ScanNet~\cite{dai2017scannet}.

\begin{figure}[t!]
    \centering
        \centering
        \includegraphics[width=0.75\linewidth]{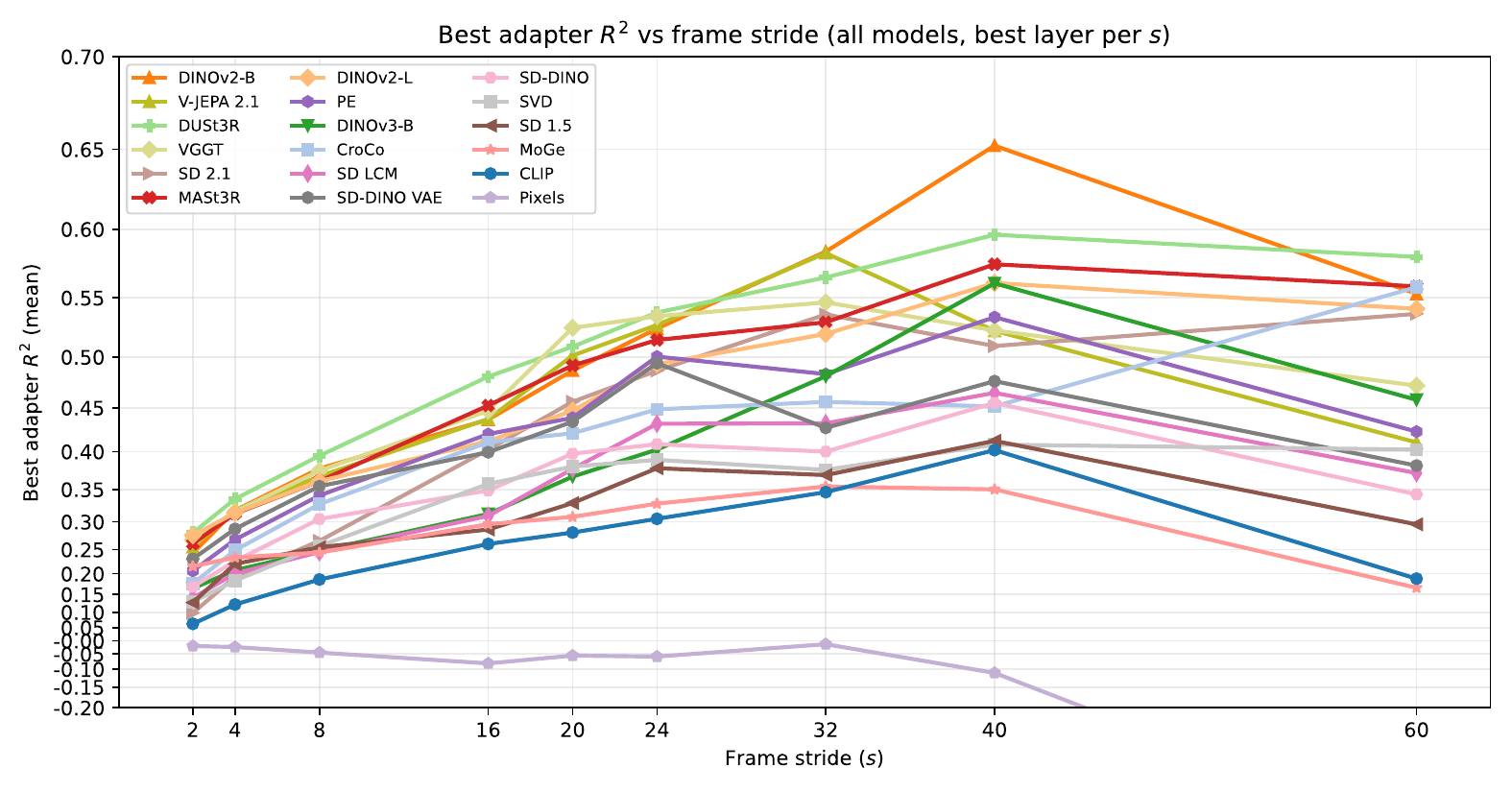}
    \caption{\textbf{Poincar\'e Adapter (M4).}
    ~Quantitative evaluation across a range of encoders and frame strides $s$. We evaluate whether changes in the latent features as observed across $s$ frames can be linearly mapped to changes in camera pose.}
    \label{fig:adapter_b} 
\end{figure}


Remarkably, by applying a lightweight Siamese MLP adapter $\varphi$, we recover strong linear equivariance. As shown in Figure~\ref{fig:adapter_b}, DINOv2 achieves strong performance, reaching test set $R^2 \approx 0.65$, averaged over 3 scenes from ScanNet. This suggests that the latent space contains a sub-manifold that is locally aligned with
\leo{$SE(3)$} and that can be accessed via a non-linear projection.
Notably, \textbf{CroCo}, despite being trained with explicit 3D tasks, underperforms DINOv2 ($R^2 \approx 0.45$).

We also note that \textit{topological} alignment already emerges without any training (\textbf{M1}), and that raw pixels remain strongly negative $R^2$ \textit{even with the adapter} (Figure~\ref{fig:adapter_b}), confirming that the decodable structure arises  in and is dependent on the properties of learned features.

Table~\ref{tab:app_dataset_summary} summarizes geometric decodability across 5 datasets. We compare the three representative foundation models (DINOv2, V-JEPA, DUSt3R)  and average decodability metrics within each dataset. Overall, we note that V-JEPA 2.1~\cite{mur2026v} and DUSt3R~\cite{wang2024dust3r} possess, on average, the strongest geometric awareness. We also remark that the difficulty of decoding the spatial structure is highly dependent on the environment, and present a scene difficulty analysis in Sec.~\ref{sec:difficulty}.
\tightpara{Zero co-visibility.} To rule out that the adapter relies on matching shared visual content, we also evaluate it on test pairs with zero co-visibility (fewer than 0.1\% of one frame's pixels reproject into the other, using depth). For example, on the 699 such pairs of the best of the ten 12-Scenes scenes, V-JEPA still reaches $R^2 = 0.80$ and DINOv2 $R^2 = 0.71$; results for all scenes are in Appendix~\ref{app:covisibility}.

\begin{table}[t!]
\centering
\resizebox{\textwidth}{!}{
\begin{tabular}{l r | rrr | rrr | rrr}
\toprule
\textbf{Dataset} & \textbf{Scenes} & \multicolumn{3}{c|}{\textbf{V-JEPA 2.1}~\cite{mur2026v}} & \multicolumn{3}{c|}{\textbf{DUSt3R}~\cite{wang2024dust3r}} & \multicolumn{3}{c}{\textbf{DINOv2}~\cite{oquab2023dinov2}} \\
& & \textbf{Max $R^2$} & \textbf{$R^2 > 0$} & \textbf{$R^2 > 0.3$} & \textbf{Max $R^2$} & \textbf{$R^2 > 0$} & \textbf{$R^2 > 0.3$} & \textbf{Max $R^2$} & \textbf{$R^2 > 0$} & \textbf{$R^2 > 0.3$} \\
\midrule
\textbf{12-Scenes}~\cite{valentin2016learning} & 10 & \textbf{0.803} & \textbf{100\%} & 70\% & 0.789 & 90\% & \textbf{80\%} & 0.682 & 80\% & 50\% \\
\textbf{ScanNet}~\cite{dai2017scannet} & 707 & \textbf{0.735} & 35\% & 4\% & 0.711 & \textbf{37\%} & \textbf{6\%} & 0.673 & 31\% & 3\% \\
\textbf{ARKitScenes}~\cite{baruch2021arkitscenes} & 133 & \textbf{0.389} & \textbf{62\%} & \textbf{2\%} & 0.371 & 56\% & 1\% & 0.305 & 58\% & 1\% \\
\textbf{TUM RGB-D}~\cite{sturm2012benchmark} & 17 & 0.337 & \textbf{41\%} & 12\% & \textbf{0.435} & \textbf{41\%} & \textbf{18\%} & 0.348 & 35\% & 12\% \\
\textbf{7-Scenes}~\cite{shotton2013scene} & 46 & 0.254 & 4\% & 0\% & 0.320 & \textbf{7\%} & \textbf{2\%} & \textbf{0.331} & 2\% & \textbf{2\%} \\
\bottomrule
\end{tabular}
}
\vspace{2mm}
\caption{\textbf{Overall Dataset Success Rates (M4).} Aggregate performance metrics across all evaluated rooms and representative model backbones. The results highlight that datasets with clean bundle-adjusted poses and dense captures (e.g., 12-Scenes) achieve near-perfect decodability, whereas on ARKitScenes and 7-Scenes $R^2$ rarely exceeds 0.3 (at most 2\% of rooms).}
\label{tab:app_dataset_summary}
\end{table}
\begin{table}[t]
\centering
\small
{
\resizebox{\linewidth}{!}{
\begin{tabular}{llccccccccc}
\toprule
\textbf{Backbone} & \textbf{Method} & \textbf{MSE ($\downarrow$)} & \textbf{$R^2$ ($\uparrow$)} & \textbf{Top-1 ($\uparrow$)} & \textbf{Hit@0.1 ($\uparrow$)} & \textbf{Hit@0.2 ($\uparrow$)} & \textbf{Hit@0.3 ($\uparrow$)} & \textbf{Hit@0.5 ($\uparrow$)} & \textbf{Rot.\ err.\ ($^\circ$, $\downarrow$)} & \textbf{Trans.\ err.\ (m, $\downarrow$)} \\
\midrule
DINOv2 & Identity & 0.578 & $-$0.323 & 0.000 & 0.090 & 0.204 & 0.302 & 0.523 & 21.2 / 21.6 & 0.32 / 0.37 \\
DINOv2 & Inv.\ Poincar\'e & 0.350 & 0.199 & \textbf{0.041} & \textbf{0.252} & \textbf{0.479} & \textbf{0.628} & \textbf{0.741} & 7.7 / 23.9 & 0.15 / 0.48 \\
DINOv2 & + Attention & 0.258 & 0.407 & 0.018 & 0.167 & 0.323 & 0.383 & 0.422 & 65.8 / 64.8 & 1.33 / 1.40 \\
\midrule
DUSt3R & Inv.\ Poincar\'e & 0.226 & 0.155 & 0.030 & 0.216 & 0.346 & 0.467 & 0.629 & 13.3 / 29.1 & 0.25 / 0.71 \\
DUSt3R & + Attention & \textbf{0.108} & \textbf{0.597} & 0.024 & 0.201 & 0.357 & 0.455 & 0.535 & 25.7 / 37.0 & 0.44 / 0.80 \\
\bottomrule
\end{tabular}
}
}
\vspace{1mm}
\caption{{\textbf{Latent Space Navigation Results (Averaged over 3 scenes, seed 0).} The Identity baseline (no navigation) calibrates task difficulty. The zero-shot Inv.\ Poincar\'e approach leads in topological retrieval (Hit@$\epsilon$) despite the Attention corrector achieving better regression metrics ($R^2$). Rot./Trans.\ err.: SO(3) geodesic angle and camera-centre distance between the retrieved and the target frame, median / mean over test pairs. }}
\label{tab:navigation_results}
\end{table}

\begin{figure}[t!]
    \centering
    \includegraphics[width=0.85\linewidth]{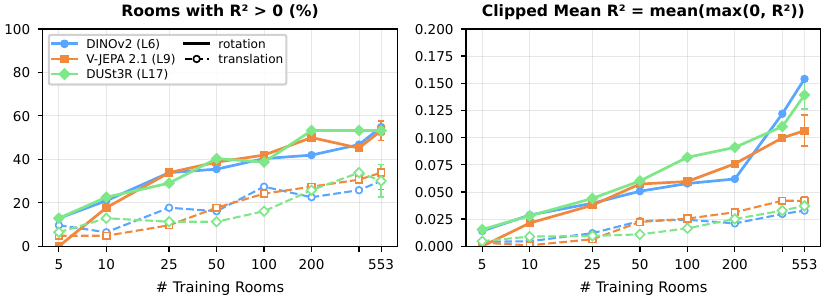}
    \caption{\textbf{Generalization and Scaling.} Zero-shot recovery of relative rotation (solid) and translation (dashed) on 62 held-out ScanNet rooms, never used for training or model selection, as a function of the number of training rooms $N$. \textbf{Left:} fraction of rooms with $R^2>0$. \textbf{Right:} clipped mean $R^2$. $N<553$: one seed; $N=553$: mean $\pm$ s.d.\ over three seeds.}
    \label{fig:scaling_curves}
\end{figure}

\tightpara{Generalization and Scaling.} To understand how geometric awareness generalizes, we trained Poincar\'e adapters on nested sets of ScanNet rooms ($N \in \{5, 10, 25, 50, 100, 200, 400, 553\}$) and evaluated them zero-shot on 62 held-out rooms. As shown in Figure~\ref{fig:scaling_curves}, generalization improves steadily with scale across multiple vision foundation models (DINOv2, V-JEPA, DUSt3R) despite their different pre-training objectives. Across all three foundation models we also observe that the rotation component generalizes consistently better than the translation one, at every $N \geq 10$, consistent with the rotation--translation asymmetry analyzed in Section~\ref{sec:analysis} (Theorem~\ref{thm:rot-trans}). For example, the fraction of rooms where the zero-shot adapter recovers a positive rotational $R^2$ grows from roughly 30--35\% at $N = 25$ to 53--55\% at $N = 553$, while the rotational clipped-mean $R^2$ roughly triples. Notably, the rotational clipped-mean curves show no clear plateau at $N = 553$ for DINOv2 and DUSt3R, suggesting that geometric awareness, and especially its rotational component, is a scalable property of vision transformers that can transfer to unseen environments.


\subsection{Latent Space Navigation}
\label{sec:navigation_results}

Building upon the linear structure revealed by the Poincaré Adapter, we test our Latent Space Navigation proof-of-concept using the intermediate features of DINOv2 (Layer 7) and the explicitly geometric DUSt3R (Layer 16). The results, averaged across multiple scenes, are summarized in \cref{tab:navigation_results}. We include an \textit{Identity} baseline that simply returns the source frame's feature ($\hat{g}_{t+s} = g_t$, i.e., no navigation at all) to calibrate the difficulty of the task. We compare the training-free ``Inv.\ Poincar\'e'' approach---which simply adds the pseudo-inverse scaled pose displacement to the current feature ($g_t + W^\dagger \Delta P$)---against a learned Attention corrector for both backbones. As shown in Table~\ref{tab:navigation_results}, the zero-shot Inv.\ Poincar\'e approach yields highly competitive retrieval, outperforming the static Identity baseline at every threshold. 
While the Attention corrector achieves better regression metrics ($R^2$), it lags behind Inv.\ Poincar\'e in physical region retrieval, suggesting that simple vector arithmetic is sufficient to move the observer to the correct spatial neighborhood. Measured in degrees and meters, the frame retrieved by Inv.\ Poincar\'e has a median rotation error of $7.7^\circ$ (DINOv2) and $13.3^\circ$ (DUSt3R), against $21.2^\circ$ for Identity, but its mean error is higher than Identity's because a minority of retrievals land far from the target. A nearest-neighbour analysis in the adapted space (Appendix~\ref{app:nn_analysis}; interpolation: Appendix~\ref{app:interpolation}) shows that the learned correctors' weaker retrieval is not because their predictions land in sparsely populated regions; retrieval instead follows whether a prediction is closer to pose-correct than to pose-incorrect frames. A full comparison including additional correctors is provided in Appendix Table~\ref{tab:navigation_results_full}. Without the adapter, the same corrector architectures trained directly on raw DINOv2 features reach Hit@0.5 of 0.570 for MLP-2 and 0.552 for Attention, vs.\ 0.523 for Identity (Table~\ref{tab:navigation_raw}).

\begin{wraptable}{r}{0.48\linewidth}
\centering
\small
\resizebox{\linewidth}{!}{\begin{tabular}{lcc}
\toprule
\textbf{Method} & \textbf{MSE ($\downarrow$)} & \textbf{Hit@0.5 ($\uparrow$)} \\
\midrule
Identity (no navigation)          & 1.853 & 0.523 \\
MLP-2 on raw features             & 1.705 & 0.570 \\
Attention on raw features         & 1.388 & 0.552 \\
\midrule
Inv.\ Poincar\'e (with adapter)    & \textbf{0.350} & \textbf{0.741} \\
\bottomrule
\end{tabular}}
\vspace{1mm}
\caption{\textbf{Navigation without the adapter (DINOv2, averaged over 3 scenes, seed 0, same test pairs as Table~\ref{tab:navigation_results}).}
The correctors are trained directly on the raw 768-D DINOv2 features.
MSE is measured in the respective feature space (raw for the first three rows, adapted for the last) and is therefore not comparable across the two spaces; Hit@0.5 is.}
\label{tab:navigation_raw}
\end{wraptable}


\textbf{Cross-Room Generalization.} We extend our navigation evaluation to 31 unseen test rooms, excluded from adapter training, to test generalizability. The zero-shot pseudo-inverse ($W^\dagger$) retrieves frames better than Identity in 58--68\% of rooms, a rate that untrained, randomly initialized adapters also reach (42--81\%; Appendix~\ref{app:cross_room_nav}). Adapting it using a Ridge-Refit (fitting only the \textit{linear} mapping component of $\Delta P \rightarrow \Delta g$ on the test room, given a frozen adapter) pushes navigation success above Identity in 87--97\% of rooms across DINOv2, DUSt3R, and V-JEPA backbones (versus 6--42\% for random adapters). This suggests that the adapter discovers a geometrically structured manifold that carries over to different environments once a linear map is refitted on the new environment.

This linear structure also enables open-loop \textbf{Multi-Step Navigation}: when $\Delta P$ is split into $S$ sub-steps and the adapter is applied iteratively, rotation stays accurate even for large displacements, while translation drift accumulates (examples in Appendix~\ref{app:multi_step_nav}).

\section{Analysis}
\label{sec:analysis}
Going back to the original Poincaré question formulated in the
introduction, our findings suggest that the answer is:
\textit{To some extent: the features of a motionless observer reflect the group of rigid camera motions, but only up to a transformation that a simple adapter can undo.}
\tightpara{Probe design is important.}
We compare the Poincar\'e adapter against four alternative formulations
for pose recovery from frozen DINOv2-B features
(Appendix~\ref{app:alternatives}, Table~\ref{tab:alternatives}).
Three design choices prove to be important. First, \textit{Lie algebra linearization}: predicting the 6D twist vector via a linear readout~$W$ outperforms $SE(3)$ matrix regression ($R^2 = 0.61$ vs.\ $0.14$), confirming that the adapted feature differences are naturally aligned
with the tangent space of $SE(3)$. Second, predicting absolute poses fails entirely ($R^2 = -3.36$), indicating that DINOv2 encodes
\textit{relational} structure between views, not scene-specific
locations. Lastly, \textit{Siamese structure}:
adapting each feature independently before subtraction
($\varphi(z_2) - \varphi(z_1)$) outperforms operating on the raw
difference ($\varphi(z_2 - z_1)$) by $\Delta R^2 = +0.13$, with
translation magnitude collapsing from $+0.60$ to $-0.08$ without it.
The Siamese design enforces \textit{homogeneity}: each feature is mapped into a common frame before comparison.
\tightpara{Why does the adapter work?}
In the Appendix Sec.~\ref{subsec:theory}, we provide a formal analysis of the conditions under
which the Poincar\'e adapter recovers SE(3) displacements. We show
(Theorem~\ref{thm:local}) that any smooth encoder whose Jacobian has
rank~6 admits a local linear readout. We then analyze \textit{global} decodability: when does a single~$W$ work across all poses? We show (Theorem~\ref{thm:global}a--b) that the
error is controlled by the \textit{curvature} of the feature manifold, as measured by the variation of the feature encoder's Jacobian across poses.
A nonlinear adapter $\varphi$ removes the feature-manifold curvature
entirely by ``unrolling'' the curved manifold into a flat coordinate
patch (Theorem~\ref{thm:global}c); a residual from the
non-commutativity of $\SE(3)$ remains, whose rotational component
vanishes for pure translational displacements and which is small
when the angular extent of the region is modest---consistent with the
observed $R^2 \approx 0.61$ where raw linear decoding fails ($R^2 < 0$).
We note that \textit{on the training} set, $R^2$ is about $0.89$ (test: $0.60$; six seeds of the configuration of Table~\ref{tab:alternatives}), suggesting that the remaining gap is due to generalization rather than a structural ceiling of the problem formulation. Validation $R^2$ is negative for every seed ($-0.41$ to $-0.75$) because the validation split pools all gaps $k \in [10, 40]$, whereas the test split uses gap~40 only; early stopping on the validation loss therefore stops training after 0--4 epochs. Measured with the SO(3) geodesic distance, the median test rotation error of this adapter is $16.2^\circ$ (mean $17.7^\circ$), vs.\ $27.5^\circ$ ($30.9^\circ$) when predicting the mean training motion.

We also study the relative difficulty of recovering different $SE(3)$ components
by decomposing the Poincar\'e adapter's output into rotation (3~DoF),
translation direction (2~DoF), and translation magnitude (1~DoF).
Across a sweep of 3{,}328 adapter configurations (varying bottleneck
dimension, learning rate, and loss type) evaluated over multiple random
seeds on frozen DINOv2-B features, a consistent pattern emerges (detailed in Appendix~\ref{sec:components_difficulty}):
rotation is reliably decodable (positive $R^2$ in 88\% of
configurations), while translation magnitude is not (positive in only
27\%, with $6\times$ the seed standard deviation).  With a well-tuned adapter,
all components achieve comparable $R^2$
(Table~\ref{tab:ablation}: $0.63$ trans, $0.59$ rot for pretrained
DINOv2-B), but translation magnitude is far more sensitive to hyperparameters and
initialization. This asymmetry has a geometric origin
(Theorem~\ref{thm:rot-trans}): rotational optical flow depends only on
pixel coordinates and focal length, while translational flow scales
as~$1/Z$ with scene depth---introducing an additional source of
Jacobian variation that makes translation harder to decode reliably.
Translation magnitude is further confounded with depth via the rescaling
invariance $(v, Z) \mapsto (\lambda v, \lambda Z)$, making $\|v\|$
unrecoverable without metric priors (Appendix~\ref{app:translation_eval}).

\tightpara{Model training -- topology is stable; geometry is fragile.}
To identify the necessary training conditions, we train 16 DINOv2-B
checkpoints \textit{from scratch} on frames taken from a single walking-tour video (using the setup introduced in \cite{UniqueLives2025}). We perform systematic ablations across the main DINOv2 hyperparameters to investigate the necessary conditions for the emergence of geometric structure in the feature extractor (Table~\ref{tab:ablation}).
We make several observations. First, mutual $k$-nn alignment (metric \textbf{M1}) is stable across recipes (M-$k$-nn~$\in [0.35, 0.42]$), while Poincar\'e $R^2$ ranges from $0.08$ to $0.47$---even four identical runs yield $R^2 \in [0.08, 0.35]$ with M-$k$-nn~$= 0.38 \pm 0.01$. The sole ablation that degrades
\textit{topology} is removing DINO self-distillation (M-$k$-nn drops to
$0.24$); reducing batch size to 128 collapses manifold formation
entirely (M-$k$-nn$\,=\,0.15$). Data scale appears to be the dominant factor for
\textit{geometry}: the pretrained model ($R^2 = 0.61$, 142M images)
outperforms the best single-video checkpoint ($R^2 = 0.47$), with the
gap concentrated in translation magnitude
(Mag $R^2 = +0.60$ vs.\ all negatives)---consistent with
Theorem~\ref{thm:rot-trans}(d): metric scale requires statistical
priors over object sizes, which diverse training may provide.

\section{Conclusion, Limitations and Future Work}
\label{sec:conclusion}
In this paper, we investigated whether vision features capture the structure of rigid motions in Euclidean space by studying the alignment between these features and the group $SE(3)$. We showed that, although no encoder reflects this structure directly, it is nevertheless possible to ``unroll'' the feature manifold  using a relatively simple adapter network in some cases. It is remarkable that visual feature changes can be mapped to pose changes by a homogeneous low-dimensional adapter that is independent of location. Interestingly, self-supervised models, while only trained through passive observation, do form a representation of relative camera motion, decodable 
when probed correctly.

One limitation of our study is that we only considered \textit{static} scenes where feature changes are only associated with camera motion, and thus do not test Poincar\'e's distinction between ``changes of position'' and ``changes of state''. A preliminary experiment on three TUM RGB-D sequences with moving people (Appendix~\ref{app:tum_dynamic}) suggests that relative rotation can remain decodable when the adapter is pre-trained on other TUM sequences and fine-tuned on the target sequence (rotation $R^2$ of 0.35--0.68, versus at most 0.06 when the same adapter is trained from scratch), while translation is not ($R^2 < 0$). It would be interesting to extend our formalism and approach to \textit{dynamic scenes}, which are governed by both camera motion as well as scene motion from underlying physical processes, aiming for an understanding of more profound spatio-temporal symmetries or conservation laws.
\section{Acknowledgements}
The authors would like to thank Dima Damen, João Carreira, Daniel Zoran, Gabrijel Boduljak, Chris Duvarney, and Andrew Zisserman for the many useful comments, discussions and feedback on this work.

\paragraph{Funding and competing interests.} This work was carried out at and funded by Google DeepMind. The authors declare no competing interests.

\bibliographystyle{unsrt}
\bibliography{biblio}


\appendix

\newpage

\appendix

\section{Appendix}

\begin{figure}[ht]
    \centering
    \includegraphics[width=0.95\linewidth]{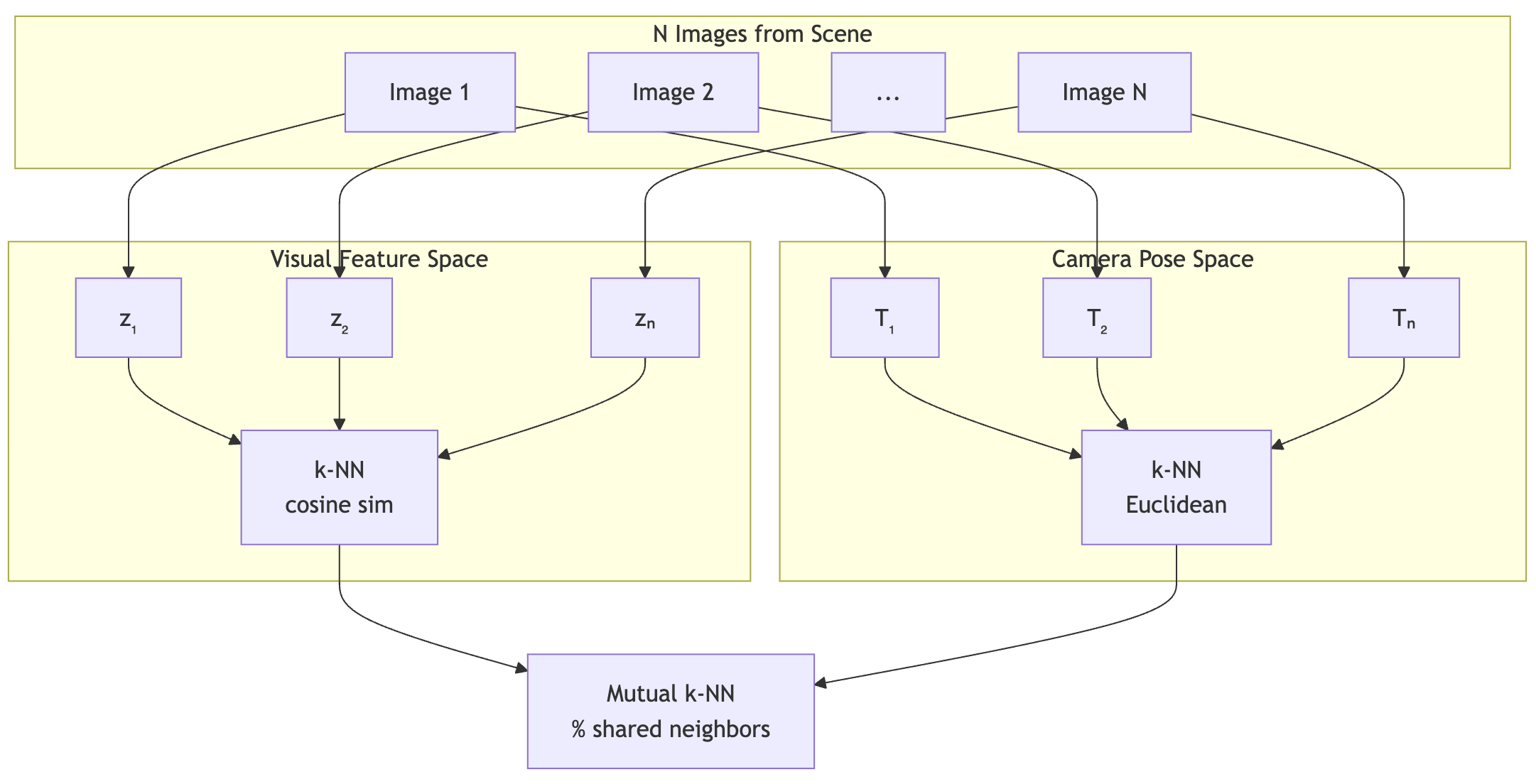}
    \caption{Overview of the mutual nearest neighbor metric \textbf{M1}. For each frame in a scene we extract its visual features, as well as its corresponding (ground truth) camera pose. We then compare the nearest neighbor graphs computed in the feature (left) and pose (right) spaces.}
    \label{fig:mutual_knn_diagram}
\end{figure}

\subsection{Metric \textbf{M1}: Mutual $k$-nn protocol}
\label{appendix:mutual_knn}
For a fixed scene we sample 256 frames, at a given stride. We compute a feature vector for each frame, and also associate to it the corresponding camera pose (see Figure~\ref{fig:scannet_examples_main} for examples of such frame pairs). We then compute the $k$ nearest-neighbors ($k=10$) in the feature space using cosine similarity, and compare these nearest neighbors to the nearest neighbors computed in the camera pose space. We use the 9D camera pose representation, comprising 3D translation and the first two columns of the rotation matrix, as defined in \cite{zhou2019continuity}, \textit{in world coordinates}.  To reduce the impact of temporal sampling, we \textit{exclude} temporal neighbors within 10 frames from the nearest neighbor computation in both visual feature and pose space. The mutual $k$-nn metric is defined as the average intersection (number of shared edges) between these nearest neighbor graphs (illustrated in Figure~\ref{fig:mutual_knn_diagram}). The metric ranges from 0 (no alignment) to 1 (perfect alignment) between visual and spatial representations. 

For video encoders that require, as input, a \textit{temporal window} of frames, we sample the encoder-native set of frames around each anchor frame and pool all features within this temporal window to obtain a single feature vector. We also experimented with keeping only the features associated with a specific anchor frame, without a noticeable difference in the results.

Since this metric requires \textit{a single vector per frame}, for encoders that contain a CLS token, we take features from that token. For others, we perform mean pooling over the patch tokens. 

\begin{figure}[ht]
    \centering
    \includegraphics[width=0.95\linewidth]{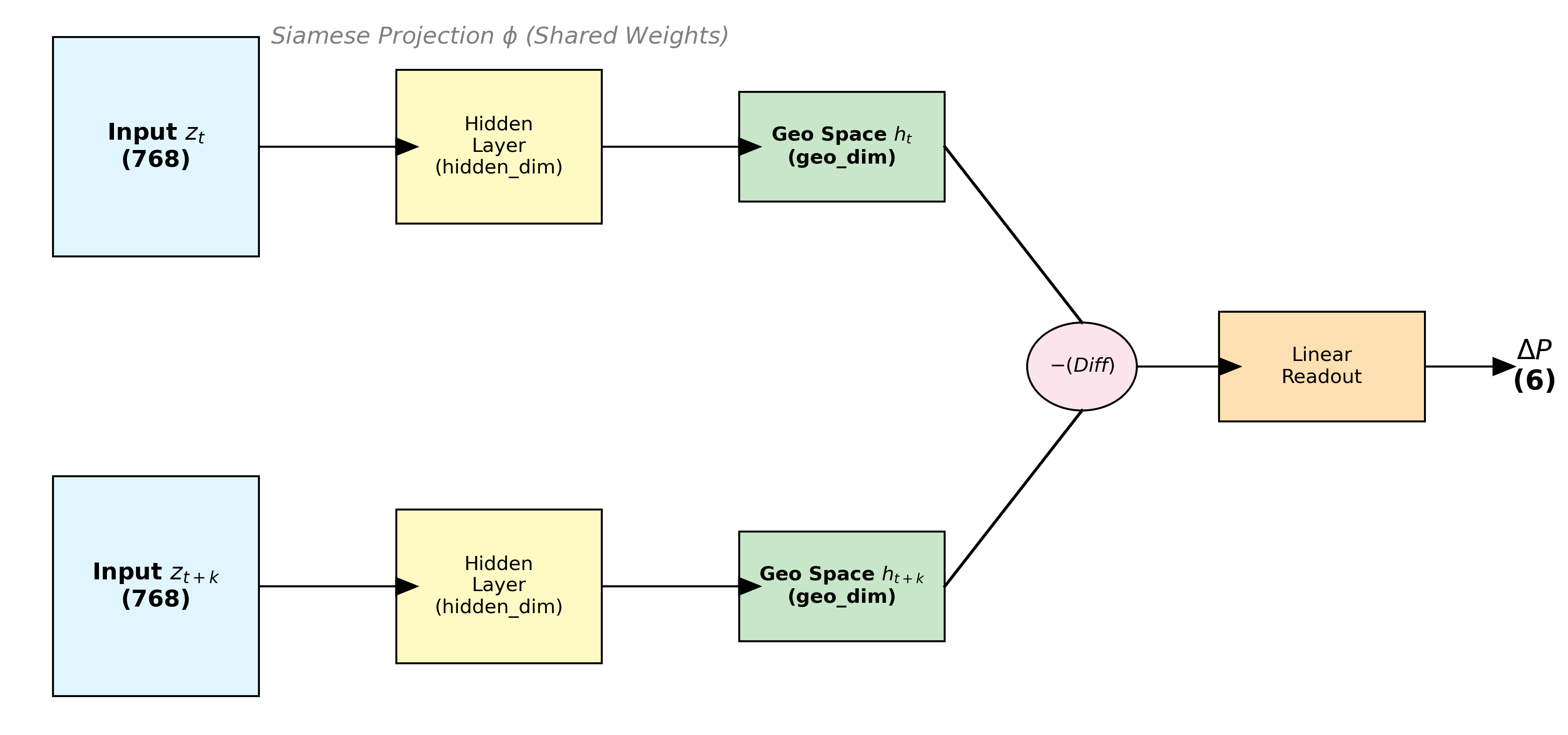}
    \caption{Poincar{\'e} Adapter architecture:  illustration of the Siamese architecture of our ``Poincaré adapter'' used in our Metric M4.}
\label{fig:adapter_architecture}
\end{figure}

\begin{figure}[ht]
    \centering
    \begin{tabular}{@{}c@{\hspace{2mm}}c@{}}
        \includegraphics[width=0.48\linewidth]{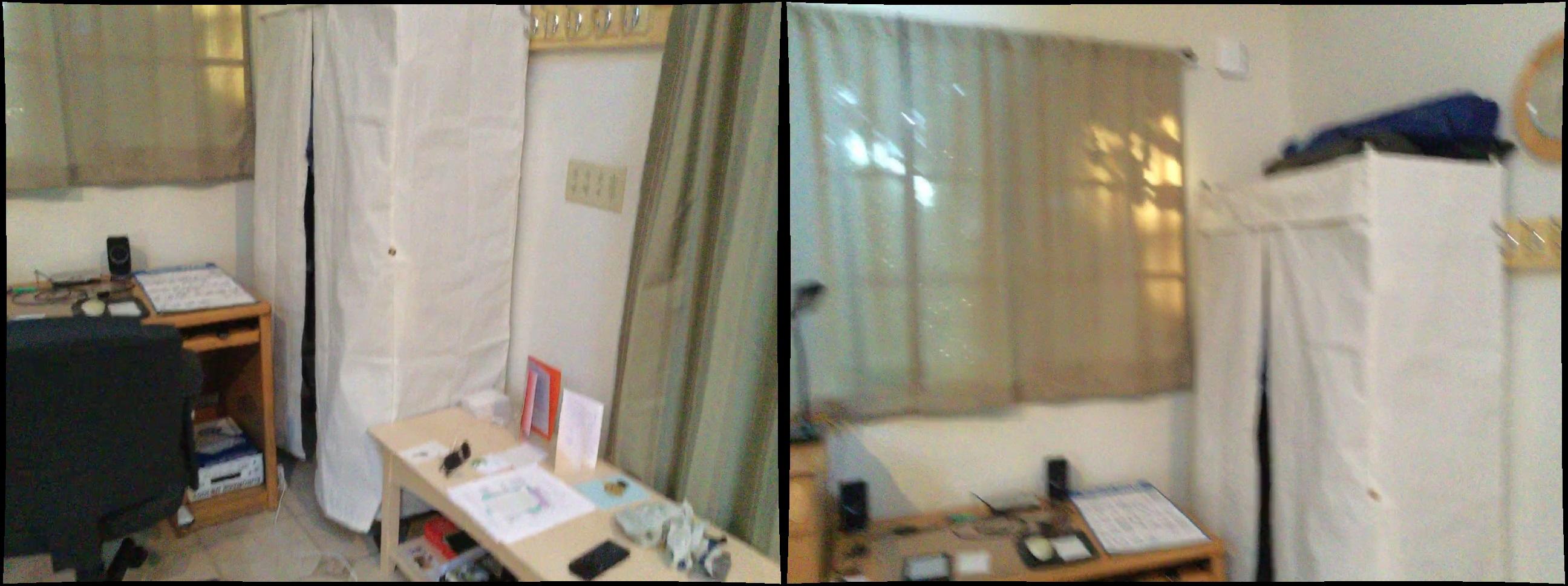} &
        \includegraphics[width=0.48\linewidth]{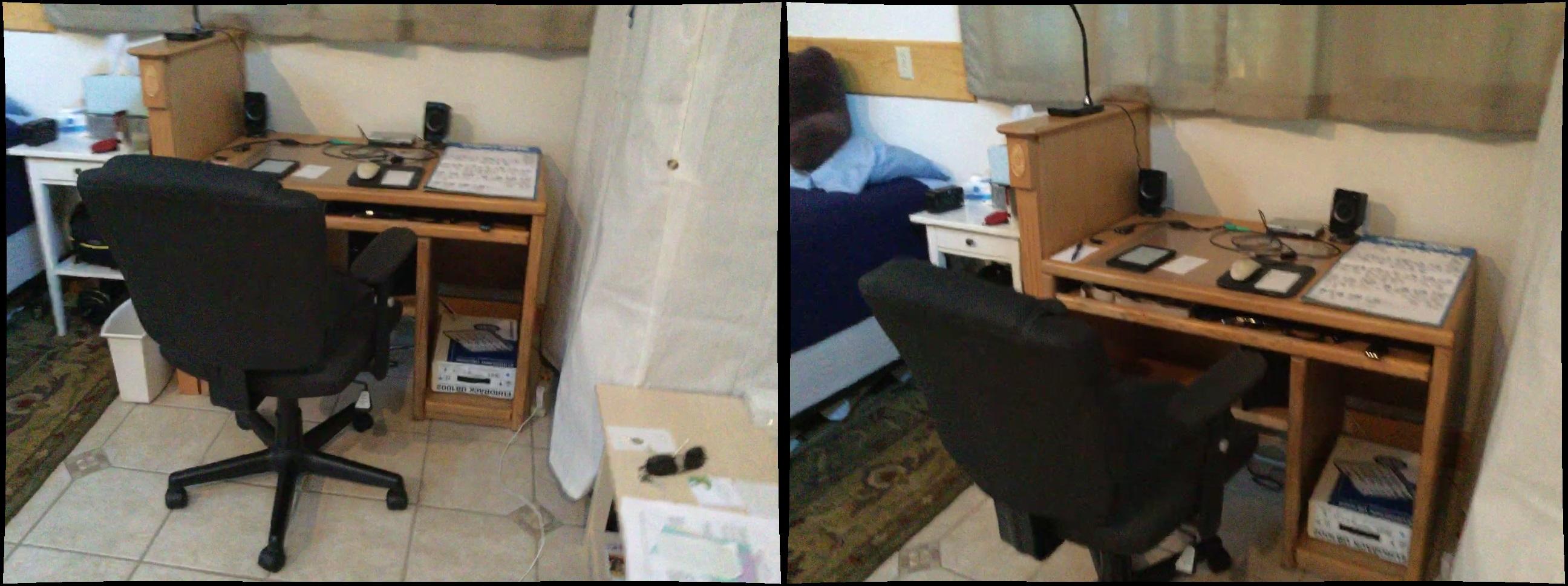} 
    \end{tabular}
    \caption{\textbf{Example frame pairs from ScanNet.} The frames within each of these two pairs are separated by a stride of $s=40$ frames, illustrating the typical viewpoint change that the Poincar\'e adapter must decode.}
    \label{fig:scannet_examples_main}
\end{figure}
\subsection{Metrics \textbf{M3} and \textbf{M4}: implementation details}
\label{app:implementation_details}

For both metrics, we construct frame pairs at a fixed temporal stride $s$ and split them using a \emph{time-based} strategy, so that no test frame appears during training. For \textbf{M3}, the first 80\% of the frames of each scene (ordered by timestamp) are used for training and the remaining 20\% for testing. For \textbf{M4}, the first 80\% are used for training, the next 10\% for validation (early stopping) and the last 10\% for testing; pairs whose two frames fall into different blocks are discarded.
For \textbf{M3} (Linear Equivariance), we fit a Ridge regression (scikit-learn, $\alpha{=}1.0$, with intercept) from the raw feature displacement $\Delta z = z_{t+s} - z_t$ to the 6D Lie-algebra pose target $\Delta P$, and report the test $R^2$ (uniform average over the 6 components). Table~\ref{tab:intrinsic_dim} reports results at the layer maximizing \textbf{M3} for each model.

To train and evaluate the Poincar\'e Adapter (\textbf{M4}), we use the same temporal train/test split. The Siamese projection network $\varphi_\theta$ is either an MLP, $\texttt{Linear}(d, h) \to \texttt{LayerNorm} \to \texttt{GELU} \to \texttt{Linear}(h, g)$ (Fig.~\ref{fig:adapter_b}, the scaling and navigation experiments; DUSt3R navigation uses a three-layer variant), or a single bias-free $\texttt{Linear}(d, g)$, so that the whole adapter is a rank-$g$ linear map (Table~\ref{tab:app_dataset_summary} and Tables~\ref{tab:alternatives}, \ref{tab:covisibility} and \ref{tab:ablation}; for TUM, 12-Scenes and 7-Scenes, a mixture of up to $K{=}4$ such linear experts, Appendix~\ref{app:moe_architecture}). It is followed by a bias-free linear readout $W \in \mathbb{R}^{6 \times g}$. Notably, this rank-constrained linear adapter reaches $R^2{=}0.61$ on DINOv2-B (Table~\ref{tab:alternatives}), close to the MLP adapter ($R^2 \approx 0.65$, Fig.~\ref{fig:adapter_b}), whereas linear regression on the raw feature displacement (\textbf{M3}) does not recover the pose (Table~\ref{tab:intrinsic_dim}). The regression targets $\Delta P$ are z-scored (zero mean, unit variance) per component on the training set before training, and predictions are un-normalized before computing the test $R^2$. The adapter is optimized with AdamW and an MSE loss, with gradient clipping at norm~1.0 and early stopping on the validation loss; Table~\ref{tab:hyperparams} lists the layers, pairs, sizes $h$ and $g$, optimizer settings, schedules and seeds of each experiment. Similarly to \textbf{M1} described above, for video encoders we use a temporal window around each anchor frame and pool  features within this window to obtain a single feature vector. We use a CLS token when available, or, mean pooled patch tokens, otherwise. The representative architecture of the ``Poincar\'e Adapter'' network is visualized in Figure~\ref{fig:adapter_architecture}.

\paragraph{Generative models.} Following the SD-DINO approach~\cite{zhang2023tale}, we obtain one feature vector per image from a diffusion model with a single denoising pass at a fixed noise level. The image is encoded by the VAE, noise is added at a fixed timestep ($t = 100$; the same noise sample for all images), and the U-Net is run once with an empty text prompt. For SD~1.5, SD~2.1 and LCM, we read out the activations of the U-Net decoder blocks, resize them to a common grid, L2-normalize them and average them over space; the layer index refers to the decoder block. For SD-DINO (U-Net), these Stable Diffusion features are concatenated with DINOv2-B patch features, with equal weights, before pooling, and the layer index refers to the DINOv2 branch; the SD-DINO (VAE) variant uses VAE-encoder features, without added noise, instead of U-Net features. For SVD, we run the first denoising step of the image-to-video pipeline with deterministic latents and pool the temporal U-Net features.
\begin{table}[ht]
\centering
\small
\resizebox{\linewidth}{!}{
\begin{tabular}{lllllll}
\toprule
\textbf{Experiment} & \textbf{Features (layer)} & \textbf{Pairs (frame gap)} & \textbf{Adapter} & \textbf{$h$ / $g$} & \textbf{AdamW (lr, wd, batch)} & \textbf{Schedule; seeds} \\
\midrule
M4 curves (Fig.~\ref{fig:adapter_b}) & best layer per $s$ & $s$ & MLP & 64--512 / 5--64 & $10^{-3}$, $10^{-4}$, 512 & $\le$50 ep., patience 5; 1 seed, best test $R^2$ of 676 configs \\
Datasets (Table~\ref{tab:app_dataset_summary}) & DINOv2-B (5,6,7) / V-JEPA 2.1 (5,6,9) / DUSt3R (16,17,18) & 4--60 & linear ($K{=}1$; $K{\le}4$ for TUM, 12-/7-Scenes) & -- / 5--64 & $10^{-3}$, $10^{-3}$, 512 & $\le$200 ep., patience 5; best test $R^2$ of 126 configs per room (ScanNet, ARKitScenes) \\
Probes (Table~\ref{tab:alternatives}) & DINOv2-B (7), CLS & train 10--40, test 40 & linear & -- / 20 & $10^{-3}$, $10^{-3}$, 512 & $\le$200 ep., patience 5; 15 seeds \\
Navigation (Tables~\ref{tab:navigation_results}, \ref{tab:navigation_raw}) & DINOv2-B (7) / DUSt3R (16) & 40 & 3-layer MLP (DUSt3R), 2-layer MLP (DINOv2) & 128 / 16 or 32 & $3\cdot10^{-3}$ (DUSt3R), $10^{-3}$ (DINOv2), $0.1$, 512 & $\le$100 ep., patience 15; seed 0 \\
Co-visibility (Table~\ref{tab:covisibility}) & per scene & 10--160 & linear, $K{\le}4$ & per scene & $10^{-3}$, $10^{-3}$, 512 & 50 ep., patience 5; seed 0 \\
Cross-room nav.\ (Table~\ref{tab:cross_room_nav}) & DINOv2-B (6), V-JEPA 2.1 (9), DUSt3R (17) & $\{4, 8, 16, 20, 32, 40, 60\}$ & MLP & 128 / 48 & $10^{-3}$, $10^{-3}$, 512 & 200 ep.; seed 0 \\
Dynamic scenes (Table~\ref{tab:tum_dynamic}) & V-JEPA 2.1 (5), mean & 1--8 & MoE, $K{=}4$ & 64 / 32 & $10^{-3}$ / $10^{-4}$ (fine-tune), $10^{-3}$, 1024 / 512 & App.~\ref{app:tum_dynamic}; 3 seeds \\
\bottomrule
\end{tabular}}
\vspace{1mm}
\caption{\textbf{Adapter settings per experiment.} $h$: hidden size; $g$: bottleneck. MoE: the mixture-of-experts adapter of Appendix~\ref{app:moe_architecture}. Linear: $\varphi_\theta = \texttt{Linear}(d, g)$; $h$ unused.}
\label{tab:hyperparams}
\end{table}
\subsection{Full Intrinsic Dimension and M3 Results for All Models}
\label{app:full_intrinsic_dim}

In Table~\ref{tab:full_intrinsic_dim}, we report the Intrinsic Dimension (TwoNN and MLE) and the \textbf{M3} metric (Max Test $R^2$) for 18 different models evaluated in our sweep, using the layer that maximized \textbf{M3} performance for each model. Our full evaluation suite also includes Depth Anything v3~\cite{lin2025da3}, Perception Encoder~\cite{bolya2025perception}, Metric3D v2~\cite{hu2024metric3d}, the generative models SD 1.5/2.1~\citep{rombach2022high,song2020denoising}, LCM DDIM~\citep{luo2023latent,rombach2022high,song2020denoising}, SD-DINO VAE and U-Net features~\citep{zhang2023tale} and SVD~\citep{blattmann2023stable}, and the video models VideoMAEv2~\cite{wang2023videomae}, RVM~\cite{zoran2025rvm} and 4DS~\cite{carreira2024scaling}.

\begin{table}[ht]
\centering
\small
\begin{tabular}{@{}l|c|cc|c@{}}
\toprule
\textbf{Model} & \textbf{Layer} & \textbf{ID (TwoNN)} & \textbf{ID (MLE)} & \textbf{M3 (Max Test $R^2$)} \\
\midrule
CLIP & 9 & 9.61 & 7.77 & 0.07 \\
CroCo & 10 & 7.66 & 5.92 & -0.42 \\
DINOv2-B & 9 & 9.15 & 6.41 & -0.28 \\
DINOv2-L (reg, pooled) & 11 & 8.00 & 5.20 & -0.11 \\
DINOv3 & 9 & 10.76 & 7.88 & -0.26 \\
DUSt3R & 19 & 7.16 & 4.40 & -0.22 \\
LCM DDIM & 3 & 9.16 & 7.72 & -0.06 \\
MASt3R & 10 & 7.37 & 5.17 & -0.23 \\
MoGe & -1 & 5.93 & 4.78 & -0.69 \\
Perception Encoder & 15 & 14.87 & 6.91 & -0.35 \\
Pixels & -1 & 5.01 & 4.03 & -5.29 \\
SD 1.5 DDIM & 3 & 9.87 & 8.35 & -0.04 \\
SD 2.1 DDIM & 3 & 10.51 & 9.15 & -0.07 \\
SD-DINO (U-Net) & 8 & 9.29 & 5.83 & -0.75 \\
SD-DINO (VAE) & 9 & 10.18 & 6.16 & -0.04 \\
SVD & 3 & 7.70 & 6.70 & -0.02 \\
VGGT-DINOv2 & 11 & 7.74 & 5.36 & -0.09 \\
V-JEPA 2.1 & 10 & 9.55 & 5.13 & 0.09 \\
\bottomrule
\end{tabular}
\caption{\textbf{Intrinsic Dimension and M3 Metric.} Full results for all 18 models evaluated on ScanNet. Most models cluster near the theoretical ID of 6, while M3 $R^2$ is negative or near-zero for all, \textbf{suggesting} that no model is natively linearly equivariant.}
\label{tab:full_intrinsic_dim}
\end{table}

\subsection{Training Objective, Data and Model Size}
\label{app:model_axes}
The models in Figure~\ref{fig:stride_sweep} differ in training objective, training data and model size at the same time. The clearest pattern follows the training objective: models trained with geometric objectives (DUSt3R, MASt3R, MoGe) align best, and among the other models, self-supervised image models with self-distillation (DINOv2, DINOv3) align better than the image--text model CLIP and the video models (VideoMAEv2, RVM, 4DS) (Section~\ref{sec:results}). Figure~\ref{fig:stride_sweep} alone cannot separate the objective from the data or the size. Our DINOv2-B training runs on a single walking-tour video (Section~\ref{sec:analysis}, Table~\ref{tab:ablation}) fix both: removing DINO self-distillation lowers mutual $k$-nn alignment from $[0.35, 0.42]$ to $0.24$, while the gap in Poincar\'e $R^2$ between the pretrained model (142M images) and the best single-video checkpoint points to data scale as the main factor for metric geometry. We did not vary model size in a controlled way.
\subsection{Dataset Summary \& Difficulty Analysis}
\label{sec:difficulty}

We summarize the different datasets and the key performance in Table~\ref{tab:multidataset}.
To understand why pose recovery varies across environments, we trained independent Poincaré adapters on 913 scenes spanning 5 diverse datasets (ScanNet, 12-Scenes, ARKitScenes, TUM RGB-D, 7-Scenes) using three representative encoders (V-JEPA 2.1, DUST3R, DINOv2-B). Consistently with the results presented in the main paper, we use a temporal 80\%/20\% train/test split separation.
From a large-scale comparison across multiple environments, we observe that scene difficulty is driven, to a significant extent by data properties, in addition to the choice of foundation model. Specifically, success hinges on three factors: (1) \textbf{Data Sufficiency}: scenes with above 3K pairs succeed 81\% of the time, explaining why dense datasets like 12-Scenes achieve 100\% success while sparse ones like 7-Scenes (often $<50$ test pairs) artificially fail ($7\%$ success). (2) \textbf{Visual Richness}: highly \textit{textured} rooms provide the necessary semantic anchors, outperforming texture-less traversals. (3) \textbf{Workspace Geometry}: bounded rooms naturally support dense view overlap and closure, yielding much higher decodability than linear hallway traversals. A summary of the dataset difficulties is provided in Table~\ref{tab:multidataset}.

\begin{table}[t!]
\centering
\small
\resizebox{\linewidth}{!}{
\begin{tabular}{lccccc}
\toprule
\textbf{Dataset} & \textbf{Scenes} & \textbf{Avg Frames/Scene} & \textbf{$R^2 > 0$ Rate} & \textbf{Best $R^2$} & \textbf{Bottleneck} \\
\midrule
\textbf{12-Scenes} & 10 & $\approx 1{,}000$ & \textbf{100\%} & 0.80 & None — perfect dataset \\
ScanNet & 707 & $200\text{--}5{,}000$ & 40\% & 0.74 & Single-scan data scarcity \\
ARKitScenes & 133 & $\approx 1{,}500$ & 66\% & 0.39 & Motion blur \\
TUM RGB-D & 17 & $200\text{--}800$ & 47\% & 0.44 & Traversals, texture-less scenes \\
\textbf{7-Scenes} & 46 & $\approx 500$ & \textbf{7\%} & 0.33 & $\approx 40$ test frame pairs — evaluation noise \\
\bottomrule
\end{tabular}
}
\vspace{1mm}
\caption{\textbf{Multi-Dataset Evaluation.} We evaluate the decodability of geometric features across 913 individual scenes. Performance varies dramatically by dataset, driven by data sufficiency, visual richness, and workspace geometry rather than foundation model capability.}
\label{tab:multidataset}
\end{table}

Note that the intrinsic dimensionality of most features typically is close to the theoretical value 6 (see Theorem~\ref{thm:intrinsic_dim} below), while the $R^2$ value for \textit{linear} equivariance Metric M3 is negative or close to zero, suggesting that \textit{no} model is natively aligned with rigid motion.

\subsection{Co-visibility Control}
\label{app:covisibility}
The adapter maps each image to a single global feature vector, independently of the other image, so it cannot match patches between the two images of a pair. To test whether it nevertheless relies on visual content shared by the two images, we evaluate it on test pairs with zero co-visibility on 12-Scenes~\cite{valentin2016learning}. For each scene and backbone, we use the adapter configuration of the 12-Scenes row of Table~\ref{tab:app_dataset_summary} and retrain it (seed 0) on pairs 10--160 frames apart, with the temporal 80/10/10 split. For every test pair, we back-project every fourth valid depth pixel of each frame into the other camera, without an occlusion test, and say that the pair has zero co-visibility if, in at least one direction, fewer than 0.1\% of these points fall inside the other image. $R^2$ is computed within each group of pairs, relative to that group's own variance.
Table~\ref{tab:covisibility} reports all ten scenes. Nine of them contain zero-co-visibility test pairs (7,338 in total). On these pairs, $R^2$ is positive in 5 of 9 scenes for V-JEPA 2.1 and in 4 of 9 for DINOv2-B. Every scene with a negative value is one where the adapter generalizes poorly to the test part of the trajectory (overall test $R^2 \le 0.12$ for V-JEPA 2.1 and $\le -0.10$ for DINOv2-B). In all eight scene--backbone cases with an overall test $R^2 \geq 0.4$, $R^2$ on zero-co-visibility pairs remains positive, although it is usually lower than the overall value. Measured with the SO(3) geodesic distance and pooled over all scenes, the median rotation error on zero-co-visibility pairs is $21.8^\circ$ (DINOv2-B) and $20.4^\circ$ (V-JEPA 2.1), compared with $68.1^\circ$ when always predicting the mean training motion. Pose can therefore be recovered without any shared visual content, although not in every scene.
\begin{table}[ht]
\centering
\small
\resizebox{\linewidth}{!}{
\begin{tabular}{lrr|ccc|ccc}
\toprule
 & & & \multicolumn{3}{c|}{\textbf{DINOv2-B}} & \multicolumn{3}{c}{\textbf{V-JEPA 2.1}} \\
 & \textbf{Pairs} & \textbf{Baseline} & \textbf{$R^2$} & \textbf{$R^2$} & \textbf{Rot.\ err.} & \textbf{$R^2$} & \textbf{$R^2$} & \textbf{Rot.\ err.} \\
\textbf{Scene} & \textbf{(0\%)} & \textbf{(0\%)} & \textbf{(all test)} & \textbf{(0\%)} & \textbf{(0\%)} & \textbf{(all test)} & \textbf{(0\%)} & \textbf{(0\%)} \\
\midrule
apt2/bed & 625 & 99.03$^\circ$ & 0.601 & 0.344 & 9.96$^\circ$ & 0.816 & 0.713 & 8.58$^\circ$ \\
apt2/kitchen & 476 & 67.10$^\circ$ & $-$0.109 & $-$1.131 & 28.34$^\circ$ & 0.122 & $-$0.934 & 39.25$^\circ$ \\
apt2/living & 540 & 35.79$^\circ$ & $-$0.273 & $-$0.631 & 24.05$^\circ$ & $-$0.153 & $-$0.434 & 25.85$^\circ$ \\
apt2/luke & 774 & 45.38$^\circ$ & $-$1.587 & $-$1.174 & 25.01$^\circ$ & $-$1.184 & $-$1.043 & 25.59$^\circ$ \\
office1/gates362 & 0 & -- & $-$2.309 & -- & -- & $-$3.574 & -- & -- \\
office1/gates381 & 339 & 41.08$^\circ$ & $-$1.122 & $-$2.600 & 32.63$^\circ$ & $-$0.966 & $-$1.003 & 12.58$^\circ$ \\
office1/lounge & 1087 & 107.72$^\circ$ & $-$0.212 & $-$0.565 & 91.14$^\circ$ & 0.466 & 0.266 & 35.94$^\circ$ \\
office1/manolis & 1784 & 102.70$^\circ$ & 0.424 & 0.090 & 23.62$^\circ$ & 0.691 & 0.615 & 24.33$^\circ$ \\
office2/5a & 1014 & 37.75$^\circ$ & 0.108 & 0.006 & 18.02$^\circ$ & 0.507 & 0.172 & 14.68$^\circ$ \\
office2/5b & 699 & 38.84$^\circ$ & 0.723 & 0.713 & 5.85$^\circ$ & 0.779 & 0.798 & 4.84$^\circ$ \\
\midrule
All scenes (pooled) & 7338 & 68.08$^\circ$ & & & 21.82$^\circ$ & & & 20.42$^\circ$ \\
\bottomrule
\end{tabular}}
\vspace{1mm}
\caption{\textbf{Zero co-visibility control on 12-Scenes.} $R^2$ over all test pairs and over the test pairs with zero co-visibility (0\%), and median SO(3) geodesic rotation error on the 0\% pairs. Baseline: median rotation error on the 0\% pairs when always predicting the mean training motion. One adapter per scene and backbone (seed 0).}
\label{tab:covisibility}
\end{table}
\subsection{Ablation: Alternative Formulations for Pose Recovery}
\label{app:alternatives}

We compare the Poincar\'e adapter formulation against four alternatives
for recovering camera pose from frozen DINOv2-B (Layer~7) features on ScanNet.
All methods share the same feature backbone, train/test splits (time-based,
3~scenes), training pairs (frame gaps 10--40; test pairs at gap 40) and adapter
capacity (hidden\_dim$=$32, geo\_dim$=$20), and all five are trained in one run.
Results are averaged over 15 random seeds (mean $\pm$ population s.d.).

\paragraph{Methods.}
Let $z_i = z(I_{g_i}) \in \mathbb{R}^{768}$ denote the frozen feature for
image~$I$ at pose $g_i \in \mathrm{SE}(3)$, and let
$\Delta P \in \mathbb{R}^6$ (Eq.~\eqref{eq:body_displacement}) be the Lie algebra target.

\begin{itemize}
  \item \textbf{Poincar\'e adapter} (baseline):
    $\hat{\Delta P} = W\bigl(\varphi(z_2) - \varphi(z_1)\bigr)$,
    where $\varphi$ is a single linear map $\mathbb{R}^d \to \mathbb{R}^g$ and $W$ is a bias-free linear readout.
    Loss: MSE on z-scored Lie algebra targets.
    \emph{Siamese}: $\varphi$ is applied independently to each feature
    before taking the difference.

  \item \textbf{Method~A} (unanchored relative):
    $\hat{M} = \psi(z_1) \cdot \psi(z_2)^{-1}$, where $\psi$ maps each
    feature to an SE(3) matrix (3D translation + 6D rotation representation),
    trained against $M_{\mathrm{gt}} = g_2^{-1} g_1$ (Eq.~\eqref{eq:body_displacement}).
    Loss: $\|\log(\hat{M}^{-1} M_{\mathrm{gt}})\|^2$ (geodesic on SE(3)).

  \item \textbf{Method~B} (absolute pose):
    $\hat{W}_i = \psi(z_i)$ predicts the absolute camera-to-world pose;
    the relative pose is recovered as
    $\hat{M} = \hat{W}_2^{-1} \cdot \hat{W}_1$,
    consistent with $M_{\mathrm{gt}} = g_2^{-1} g_1$.
    Loss: geodesic on SE(3) for each absolute pose independently,
    i.e., $\sum_i \|\log(\hat{W}_i^{-1} g_i)\|^2$.
    Note that the composition formula is used only at evaluation time;
    the training signal is purely on absolute poses.

  \item \textbf{Method~C} (SE(3) MLP on $\Delta\varphi$):
    $\hat{M} = \mathrm{MLP}_{\mathrm{SE3}}\bigl(\varphi(z_2) - \varphi(z_1)\bigr)$,
    using the same Siamese $\varphi$ but outputting an SE(3) matrix instead
    of a Lie algebra vector.
    Loss: geodesic on SE(3).

  \item \textbf{Method~D} (non-Siamese difference):
    $\hat{\Delta P} = W \cdot \varphi(z_2 - z_1)$, where $\varphi$ operates
    on the raw feature difference rather than on individual features.
    Loss: MSE on z-scored Lie algebra targets (same as baseline).
\end{itemize}

\paragraph{Results.}
Table~\ref{tab:alternatives} summarizes the main results of this ablation.

\begin{table}[h!]
\centering
\small
\begin{tabular}{@{}lrrrrr@{}}
  \toprule
  Method & Overall $R^2$ & Trans.\ $R^2$ & Rot.\ $R^2$ & Dir.\ $R^2$ & Mag.\ $R^2$ \\
  \midrule
  \textbf{Poincar\'e adapter} & $\mathbf{0.61 \pm 0.04}$ & $\mathbf{0.63}$ & $\mathbf{0.59}$ & $\mathbf{0.41}$ & $\mathbf{0.60}$ \\
  D: $\varphi(\Delta z)$ non-Siamese & $0.48 \pm 0.07$ & $0.44$ & $0.53$ & $0.08$ & $-0.08$ \\
  C: SE(3) MLP on $\Delta\varphi$ & $0.14 \pm 0.25$ & $0.34$ & $-0.06$ & $0.30$ & $-1.32$ \\
  A: Unanchored relative & $-8.27 \pm 9.96$ & $-1.31$ & $-15.23$ & $-0.68$ & $-1.07$ \\
  B: Absolute pose & $-3.36 \pm 4.28$ & $-1.44$ & $-5.27$ & $-1.26$ & $-3.80$ \\
  \bottomrule
\end{tabular}
\caption{\textbf{Probe Formulations for SE(3) Recovery} from frozen DINOv2-B (Layer~7) features (all five methods trained in one run; mean $\pm$ population s.d.\ over 15 seeds; Dir.\ $R^2$ for B and C averages 11 of the 15 seeds). The Siamese Poincar\'e adapter ($R^2{=}0.61$) substantially outperforms all alternatives, including absolute pose prediction (which fails) and non-Siamese variants.}
\label{tab:alternatives}
\end{table}

Three findings emerge.
\textbf{(1)~Relative changes over absolute poses.}
Method~B, which predicts absolute camera pose from individual features,
fails catastrophically ($R^2 = -3.36$), while all relative-change methods except Method~A
achieve positive $R^2$ on at least some components.
We attribute this failure to the fact that DINOv2 features do not encode
scene-specific absolute camera coordinates---they encode
\emph{relational} structure between views.
Even a hypothetically perfect composition formula cannot recover
meaningful relative poses from absolute predictions that do not
correlate with the true camera positions.
\textbf{(2)~Linearization helps.}
Methods~A and~C operate on SE(3) directly rather than on the flat Lie algebra
$\mathfrak{se}(3)$, yet both perform worse than the linearized Poincar\'e adapter.
Method~C uses the same Siamese architecture and differs only in replacing
the linear readout~$W$ with an MLP outputting SE(3) matrices, yet its $R^2$
drops from $0.61$ to $0.14$.
This suggests that the feature differences $\varphi(z_2) - \varphi(z_1)$
are naturally aligned with the Lie algebra structure, and that the
linearization $\Delta P \approx W \cdot \Delta\varphi(z)$ is not a lossy
approximation but rather matches the geometry of the representation.
\textbf{(3)~Siamese structure matters.}
Method~D uses the same linear readout and loss as the baseline but applies
$\varphi$ to the raw difference $z_2 - z_1$ instead of computing
$\varphi(z_2) - \varphi(z_1)$.
This seemingly minor change drops $R^2$ from $0.61$ to $0.48$, with the
translation magnitude component collapsing from $0.60$ to $-0.08$.
The Siamese architecture, in which $\varphi$ normalizes each feature
independently before subtraction, is essential for the pose signal to be
linearly decodable from the difference.

\subsection{Pose Recovery Components Difficulty Analysis}
\label{sec:components_difficulty}
To further elaborate on our key metric \textbf{M4} results and analyze the difficulty of recovering different components of camera pose, we decompose the Poincar\'e adapter's overall $R^2$ into three major parts:
rotation (3~DoF), translation direction (2~DoF), and translation magnitude (1~DoF),
where direction and magnitude together form the full translation (3~DoF).

We assess whether there is a difference in  the difficulty of recovering these components. To this end, we conducted a hyperparameter sweep training different Poincaré adapters to optimize metric \textbf{M4} following the protocol in \ref{app:implementation_details}. We trained 3{,}328 unique adapter configurations evaluated with up to 20~random seeds each, totaling 38{,}929 runs.
All runs use frozen DINOv2-B features (Layer~7, CLS token, $d = 768$) extracted
from 3~ScanNet scenes (${\sim}1{,}200$ frames total, stride~2).
The adapter consists of a Siamese projection $\varphi$
followed by a bias-free linear readout~$W$;
we sweep over bottleneck dimension $\texttt{geo\_dim} \in \{10, 20, 40, 80, 160, 320\}$,
number of mixture-of-experts heads $K \in \{1, 2, 4\}$,
learning rate $\texttt{lr} \in [3{\times}10^{-4},\, 3{\times}10^{-2}]$,
and loss type $\in \{\text{MSE},\, \text{cosine},\, \text{rot-only}\}$.
Evaluation pairs are sampled at a fixed temporal gap (stride) of $s = 40$ frames,
and $R^2$ is computed on a held-out time-based test split.

\paragraph{Difficulty hierarchy.}
Table~\ref{tab:success} reports the fraction of adapter configurations achieving
positive $R^2$ (better than predicting the constant mean) and $R^2 > 0.5$,
aggregated over all runs.

\begin{table}[ht]
\centering
\small
\begin{tabular}{@{}lrrrr@{}}
  \toprule
  Component & Best $R^2$ & \% configs ${>}\,0$ & \% configs ${>}\,0.5$ & Seed $\sigma$ \\
  \midrule
  Rotation           & $0.813$ & $88.1\%$           & $11.2\%$             & $0.17$ \\
  Translation        & $0.730$ & $71.7\%$           & $7.0\%$              & $0.32$ \\
  Trans.\ direction  & $0.665$ & $60.5\%$           & $3.9\%$              & $0.30$ \\
  Trans.\ magnitude  & $0.725$ & $\mathbf{26.7\%}$  & $6.7\%$              & $\mathbf{1.03}$ \\
  \bottomrule
\end{tabular}
\caption{\textbf{Success Rates and Seed Stability} across 38{,}929 Poincar\'e adapter runs (DINOv2-B, Layer~7, ScanNet).
Seed $\sigma$ is the average within-config standard deviation. Rotation is reliably decodable (88\% positive $R^2$), while translation magnitude is not (27\%, with $6\times$ the seed standard deviation).}
\label{tab:success}
\end{table}

A consistent ordering emerges:
$R^2_{\text{rot}} > R^2_{\text{trans}} > R^2_{\text{dir}} > R^2_{\text{mag}}.$
Rotation is positive in 88\% of configurations; translation magnitude in only 27\%.
Moreover, translation magnitude has \textbf{6$\times$ the seed standard deviation} of rotation
($\sigma = 1.03$ vs.\ $0.17$), meaning that even the configurations that achieve
positive $R^2_{\text{mag}}$ in one seed often fail in another.
While individual configurations occasionally achieve $R^2_{\text{mag}}$ up to
$0.72$, the metric is not \emph{reliably} linearly decodable: across all runs,
the mean $R^2_{\text{mag}}$ is $-3.30$.

\paragraph{Why rotation is easier.}
The difficulty hierarchy has a precise geometric explanation rooted in
optical flow and formalized in Theorem~\ref{thm:rot-trans} below. Intuitively, rotational flow is depth-independent, whereas the translational Jacobian depends on depth map (see Theorem~\ref{thm:rot-trans} for details).


\paragraph{Joint training regularizes rotation.}
Perhaps counter-intuitively, training the adapter to predict all 6~DoF jointly
produces better rotation $R^2$ than predicting rotation alone.
Under matched conditions (same architecture, 5~seeds):
\begin{center}
\small
\begin{tabular}{@{}lrrr@{}}
  \toprule
  Objective & Rot $R^2$ mean & Rot $R^2$ std & Rot $R^2$ max \\
  \midrule
  Full 6D (rot + trans)  & $\mathbf{0.613}$ & $0.19$ & $\mathbf{0.813}$ \\
  Rot-only (3D readout)  & $0.552$ & $0.05$ & $0.634$ \\
  \bottomrule
\end{tabular}
\end{center}
One possible interpretation is that the translation objective acts as a multi-task regularizer for the shared feature projection~$\varphi$, reducing the possibility of overfitting.

\subsection{Evaluating Translation: Direction and Magnitude}
\label{app:translation_eval}
We split translation $R^2$ into direction (scale-invariant) and magnitude (scale-sensitive, in meters); the magnitude term tests whether metric change can be regressed from feature changes. Metric reconstruction from a single frame is ill-defined because of scale ambiguity, and our protocol does not require it. (a)~The adapter is trained with supervision and evaluated on the same scene types (e.g., ScanNet, ARKitScenes), which gives some scale calibration through scene-consistent supervision and evaluation. (b)~The protocol does not test metric memorization; it tests whether feature changes are approximately equivariant to camera motion, e.g., whether moving forward by 1\,m corresponds to a predictable change in feature space, which is weaker than full metric reconstruction. We attribute the magnitude recovered from pretrained DINOv2 to its training on 142M diverse images, which expose the model to objects at many scales. DINOv2-B trained only on walking-tour video (Section~\ref{sec:analysis}) does not transfer this to indoor scenes: translation-magnitude $R^2$ is negative for every such checkpoint, against $+0.60$ for the pretrained model.
\subsection{Theoretical Analysis}
\label{subsec:theory}
Before proceeding, we note that even without considering the specifics of the visual feature extractor, under fairly general conditions, the acquisition process itself imposes theoretical constraints on the structure in the feature domain. To state this formally, we frame the visual feature acquisition as a continuous generative process over time, and express these constraints with the theorem below (which is derived from basic principles, but which we state explicitly for clarity and completeness):

\begin{restatable}{theorem}{IntrinsicDimThm}
\label{thm:intrinsic_dim}
Let $\mathcal{S}$ be a static scene, and let $\mathcal{U} \subseteq SE(3)$ be an open set of valid camera poses. Consider a visual feature acquisition process $\mathcal{P}(t)$ parameterized by time $t \in \mathbb{R}$, and defined as:
\begin{align}
    \mathcal{P}(t) = (\mathcal{F} \circ P \circ C)(t)
\end{align}
where $C: \mathbb{R} \to \mathcal{U}$ maps time to a camera pose, $P: \mathcal{U} \to \mathcal{I}$ is the rendering mapping from a camera pose to an image in the space of pixel values $\mathcal{I}$, and $\mathcal{F}: \mathcal{I} \to \mathbb{R}^d$ is a visual feature extractor. Then:
\begin{enumerate}
    \item $P$ is a well-defined function, with a unique image for any given pose, and any temporal trajectory $\mathcal{P}(t)$ is confined to the scene-specific latent set $\mathcal{M}_{\mathcal{S}} = (\mathcal{F} \circ P)(\mathcal{U})$.
    \item Assuming $\Phi = \mathcal{F} \circ P$ is a smooth mapping, the intrinsic (Hausdorff) dimensionality of $\mathcal{M}_{\mathcal{S}}$ is at most $6$.
    \item If $\Phi = \mathcal{F} \circ P$ is a smooth \emph{embedding}---meaning it is an injective immersion that maps homeomorphically onto its image---then $\mathcal{M}_{\mathcal{S}}$ is a regular smooth submanifold of $\mathbb{R}^d$ with  intrinsic dimension \emph{of exactly} $6$.
\end{enumerate}
\end{restatable}
\begin{proof}
We address each claim sequentially:
\begin{enumerate}
    \item For a strictly static scene $\mathcal{S}$, the geometry, materials, and illumination of the environment are invariant over time. Thus, time $t$ influences the visual observation solely through the trajectory $C(t)$, making $P: \mathcal{U} \to \mathcal{I}$ a well-defined mapping that assigns a unique, repeatable image to any specific pose. Moreover, by definition, the camera trajectory is restricted to valid poses, meaning $C(t) \in \mathcal{U}$ for all $t \in \mathbb{R}$. Therefore, the composed mapping evaluated at any time $t$ yields $\mathcal{P}(t) = \mathcal{F}(P(C(t))) \in (\mathcal{F} \circ P)(\mathcal{U}) = \mathcal{M}_{\mathcal{S}}$. The continuous, time-parameterized sequence of features $\mathcal{P}(\mathbb{R})$ is thus a 1-dimensional curve constrained entirely within the bounds of the spatial set $\mathcal{M}_{\mathcal{S}}$.
    \item Let $\Phi = \mathcal{F} \circ P$. The domain of valid poses $\mathcal{U}$ is an open subset of the Lie group $SE(3)$, which is a smooth manifold of dimension $6$. Because the composition $ \Phi: \mathcal{U} \to \mathbb{R}^d $ is assumed to be a smooth mapping, it is continuously differentiable and therefore locally Lipschitz continuous. By standard results in geometric measure theory \cite{federer1969geometric}, locally Lipschitz mappings do not increase the Hausdorff dimension of a set. Therefore, $\dim_{\mathcal{H}}(\mathcal{M}_{\mathcal{S}}) = \dim_{\mathcal{H}}(\Phi(\mathcal{U})) \leq \dim_{\mathcal{H}}(\mathcal{U}) = 6$.
    \item By definition, a smooth embedding $\Phi$ is an immersion (its differential $d\Phi$ has a full rank of $6$ everywhere) that is injective and maps homeomorphically onto its image. Under these assumptions, standard differential topology \cite{lee2012smooth} guarantees that the image of an embedded manifold is a regular smooth submanifold, diffeomorphic to the domain. Because diffeomorphisms preserve dimensionality, $\mathcal{M}_{\mathcal{S}} = \Phi(\mathcal{U})$ is a regular smooth submanifold of $\mathbb{R}^d$ with an exact intrinsic dimension of $6$.
\end{enumerate}
\end{proof}

\subsubsection{Geometric Decodability from Features -- Setup and Notation}

Throughout the theoretical analysis, relative pose is expressed in the \emph{body frame}: $\xi = \log(g_1^{-1} g_2)$, which describes the
displacement as seen from~$g_1$. Eq.~\eqref{eq:body_displacement} uses the inverse transform, whose logarithm is $-\xi$; the sign does not affect the analysis. Note that for the Poincaré adapter we define the target motion $\Delta P$ in the \textbf{local camera coordinate frame} (the body frame). In this frame, a translation vector corresponds to an agent-centric action (e.g., ``move forward 1 meter'') rather than a change in absolute map coordinates (e.g., ``move North''). This choice is essential for \textbf{homogeneity}: a specific visual change (like the radial expansion of optical flow) should always correspond to the same displacement vector $\Delta P$, regardless of where the agent is located in the world or which direction it is facing.

In Appendix~\ref{app:alternatives}, all methods are trained and scored against the same body-frame target, Eq.~\eqref{eq:body_displacement}.

\textbf{Rendering model.}
Fix a 3D scene~$s$. A camera with pose $g \in \SE(3)$ renders an image
$I_g = \mathcal{R}(g, s) \in \R^{H \times W \times 3}$.
We write the Lie algebra decomposition
$\se(3) = \so(3) \oplus \R^3$, so a tangent vector
$\xi = (\omega, v)$ encodes infinitesimal rotation
$\omega \in \R^3$ and translation $v \in \R^3$.

\textbf{Encoder.}
Let $z\colon \R^{H \times W \times 3} \to \R^d$ be a feature encoder
(e.g., a vision transformer). Define the \emph{posed feature map}:
\[
  f\colon \SE(3) \to \R^d, \quad f(g) = z\bigl(\mathcal{R}(g, s)\bigr).
\]

\begin{definition}[Poincar\'e adapter]
We say $f$ admits a \emph{Poincar\'e adapter at~$g_0$} if there exist
a $C^1$ diffeomorphism $\varphi\colon U \to V$
(where $U \ni f(g_0)$ is open in $\R^d$)
and a matrix $W \in \R^{6 \times d}$ such that for $g$ near $g_0$:
\[
  W \cdot \bigl(\varphi(f(g)) - \varphi(f(g_0))\bigr)
  = \log(g_0^{-1} g) + O\!\bigl(\|\log(g_0^{-1} g)\|^2\bigr).
\]
We say $f$ admits a \emph{linear Poincar\'e adapter}
if the above holds with $\varphi = \id$.
\end{definition}

\subsubsection{Local Geometric Decodability}

\begin{theorem}[Local Geometric Decodability]
\label{thm:local}
Let $f\colon \SE(3) \to \R^d$ with $d \geq 6$ be the posed feature map
of a $C^1$ encoder~$z$. Suppose:
\begin{enumerate}
  \item[\textbf{(C1)}] \textbf{Smoothness.}
    $f$ is $C^1$ on a neighborhood of $g_0 \in \SE(3)$.
  \item[\textbf{(C2)}] \textbf{Local discriminability.}
    The differential
    $\mathrm{d}f_{g_0}\colon \se(3) \to \R^d$ has rank~$6$.
\end{enumerate}
Then $f$ admits a linear Poincar\'e adapter at~$g_0$.
Explicitly, there exists $W \in \R^{6 \times d}$ such that:
\[
  W \cdot \bigl(f(g) - f(g_0)\bigr)
  = \log(g_0^{-1} g)
    + O\!\bigl(\|\log(g_0^{-1} g)\|^2\bigr).
\]
\end{theorem}

\begin{proof}
Define $F\colon \se(3) \cong \R^6 \to \R^d$ by
$F(\xi) = f(g_0 \cdot \exp(\xi)) - f(g_0)$.
Then $F(0) = 0$ and the Jacobian at the origin is
$J_F(0) = \mathrm{d}f_{g_0}$
(since $\mathrm{d}\exp_0 = \id$ on $\se(3)$).

By~(C2), $J_F(0) \in \R^{d \times 6}$ has rank~$6$,
so it admits a left inverse:
there exists $W \in \R^{6 \times d}$ with $W \cdot J_F(0) = I_6$.
By Taylor's theorem:
\[
  W \cdot F(\xi)
  = W \cdot J_F(0) \cdot \xi + O(\|\xi\|^2)
  = \xi + O(\|\xi\|^2).
\]
Since $g = g_0 \cdot \exp(\xi)$ implies $\xi = \log(g_0^{-1} g)$,
this gives the result.
\end{proof}

\begin{remark}
Theorem~\ref{thm:local} is essentially trivial---it follows
from Taylor's theorem and the existence of a left inverse.
The content is in recognizing that~(C2) is the \emph{only} condition
needed for local decodability: any smooth, locally discriminative
encoder automatically admits a Poincar\'e adapter.
The hard question is:
\textbf{when does a single~$W$ work across a range of poses~$g_0$?}
This is the content of Theorem~\ref{thm:global}.
\end{remark}

\subsubsection{Global Linear Decodability and Its Obstruction}

The local readout $W(g_0)$ from Theorem~\ref{thm:local} depends on
$g_0$ because the Jacobian $\mathrm{d}f_{g_0}$ varies with pose.
Define the \textbf{Jacobian field}:
\[
  J\colon \SE(3) \to \R^{d \times 6}, \quad J(g) = \mathrm{d}f_g.
\]

\begin{theorem}[Global Geometric Decodability]
\label{thm:global}
Let $f\colon \SE(3) \to \R^d$ be $C^2$ with
$\rank(\mathrm{d}f_g) = 6$ for all~$g$ in a connected compact
geodesically convex region $K \subset \SE(3)$
(i.e., every minimizing geodesic between points in~$K$ lies in~$K$).

\begin{enumerate}
\item[\textbf{(a)}] \textbf{Linear Poincar\'e adapter.}
Fix any $W \in \R^{6 \times d}$ that is a left inverse of $J(g_*)$
for some base pose $g_* \in K$ (i.e., $W \cdot J(g_*) = I_6$). Then:
\[
  W \cdot \bigl(f(g_2) - f(g_1)\bigr)
  = \log(g_1^{-1} g_2) + \varepsilon(g_1, g_2)
\]
with first-order error
\[
  \|\varepsilon(g_1, g_2)\|
  \;\leq\;
  \Bigl(\sup_{g \in K} \|W \cdot J(g) - I_6\|\Bigr)
  \cdot \|\log(g_1^{-1} g_2)\|
  + O\!\bigl(\|\log(g_1^{-1} g_2)\|^2\bigr).
\]
The first-order error vanishes if and only if $W \cdot J(g)$
is constant over~$K$---i.e., the Jacobian field is
\textbf{constant modulo $\ker W$}.

\item[\textbf{(b)}] \textbf{Obstruction (manifold curvature).}
The first-order error is controlled by the \textbf{curvature} of the
feature manifold $f(K) \subset \R^d$:
\[
  \sup_{g \in K} \|W \cdot J(g) - I_6\|
  \;\leq\;
  \diam(K) \cdot \sup_{g \in K} \|W \cdot \nabla J(g)\|,
\]
where $\nabla J$ denotes the covariant derivative of~$J$ along~$K$,
and $\diam(K)$ is the diameter of~$K$ in $\SE(3)$.

\item[\textbf{(c)}] \textbf{Nonlinear adapter and intrinsic curvature.}
Suppose further that the feature map $f$ is globally injective on $K$ (hence an embedding). 
Without loss of generality, choose the world coordinate frame such that the identity 
pose $I \in K$. There exists a $C^1$ diffeomorphism $\varphi\colon U \to V$ 
(where $U, V \subset \R^d$ are open neighborhoods of $f(K)$ and $\varphi(f(K))$, respectively) 
and a readout matrix $W \in \R^{6 \times d}$ such that for all $g_1, g_2 \in K$, the Siamese 
adapter yields:
\[
  W \cdot \bigl(\varphi(f(g_2)) - \varphi(f(g_1))\bigr)
  = \log(g_1^{-1} g_2) + \varepsilon_{\mathrm{Lie}}(g_1, g_2)
    + O\!\bigl(\|\log(g_1^{-1} g_2)\|^2\bigr).
\]
The nonlinear adapter flattens the extrinsic curvature of the neural feature 
manifold (removing the Jacobian variation). The remaining first-order 
error, $\varepsilon_{\mathrm{Lie}}$, is strictly intrinsic to the non-commutative geometry 
of $\SE(3)$, governed by the inverse differential of the exponential map
$(\mathrm{d}\exp_X)^{-1}$. Its leading term is the Lie bracket, and it
is bounded by:
\[
  \|\varepsilon_{\mathrm{Lie}}(g_1, g_2)\| 
  \leq C(R_K) \cdot \|\log(g_1^{-1} g_2)\|,
\]
where $C(R_K) \to 0$ as $R_K \to 0$, with leading behavior
$C(R_K) = \tfrac{1}{2} C_{\mathrm{Lie}}\, R_K + O(R_K^2)$,
where $R_K = \max_{g \in K} \|\log(g)\|$ is the radius of $K$
in the Lie algebra.
\end{enumerate}
\end{theorem}

\begin{proof}[Proof of~(a)]
Fix any~$W$ that is a left inverse of $J(g_*)$ for some $g_* \in K$.
By the fundamental theorem of calculus along a geodesic $\gamma$ from
$g_1$ to $g_2$ with body velocity
$\xi = \log(g_1^{-1} g_2)$:
\[
  f(g_2) - f(g_1)
  = \biggl(\int_0^1 J(\gamma(t))\, dt\biggr) \cdot \xi.
\]
Applying~$W$:
\[
  W \cdot \bigl(f(g_2) - f(g_1)\bigr)
  = \xi
    + \biggl(\int_0^1 \bigl(W \cdot J(\gamma(t)) - I_6\bigr)\, dt\biggr)
      \cdot \xi.
\]
The error term satisfies:
\[
  \biggl\|\int_0^1 \bigl(W \cdot J(\gamma(t)) - I_6\bigr)\,dt
    \cdot \xi\biggr\|
  \;\leq\;
  \sup_{g \in K} \|W \cdot J(g) - I_6\| \cdot \|\xi\|.
\]
The error vanishes to first order iff
$W \cdot J(g) = I_6$ for all $g \in K$.
\end{proof}

\begin{proof}[Proof of~(b)]
Since $W \cdot J(g_*) = I_6$, for any $g \in K$:
\[
  \|W \cdot J(g) - I_6\|
  = \|W \cdot (J(g) - J(g_*))\|
  \leq \|W\| \cdot \|J(g) - J(g_*)\|.
\]
By the mean value theorem on the Lie group,
$\|J(g) - J(g_*)\| \leq \sup_\gamma \|\nabla J\| \cdot d(g, g_*)$.
Taking the supremum over~$K$ and noting
$d(g, g_*) \leq \diam(K)$ gives the result
(absorbing $\|W\|$ into the $\nabla J$ term).
\end{proof}

\begin{proof}[Proof of~(c)]
Since $f$ is a $C^1$ immersion and globally injective on the compact set $K$, the image 
$\mathcal{M} = f(K)$ is an embedded 6-dimensional submanifold of $\R^d$. By the Tubular 
Neighborhood Theorem, there exists an open neighborhood $U$ of $\mathcal{M}$ and a smooth 
projection $\pi\colon U \to \mathcal{M}$. 

Define the Lie algebra chart around the identity $\psi(g) = \log(g) \in \se(3) \cong \R^6$. 
We first define a map on the manifold $\hat\varphi = \psi \circ f^{-1} \circ \pi \colon U \to \R^6$. 
To make this a full diffeomorphism onto its open image in $\R^d$, we append the normal bundle 
coordinates: $\varphi(x) = (\hat\varphi(x),\, x - \pi(x)) \in \R^6 \times \R^{d-6} \cong \R^d$.

Setting $W = [I_6 \mid 0]$, the adapter exactly evaluates differences in the absolute 
Lie algebra coordinates (linearizing around $0$):
\[
  W \cdot \bigl(\varphi(f(g_2)) - \varphi(f(g_1))\bigr)
  = \log(g_2) - \log(g_1).
\]
Let $X = \log(g_1)$ and let $\xi = \log(g_1^{-1} g_2)$ be the true relative body velocity. 
Observe the exact group composition: $g_2 = g_1 (g_1^{-1} g_2) = \exp(X)\exp(\xi)$. 
By the Baker--Campbell--Hausdorff (BCH) formula, the full linear-in-$\xi$
expansion is:
\[
  \log\bigl(\exp(X)\exp(\xi)\bigr) = X + (\mathrm{d}\exp_X)^{-1}(\xi) + O(\|\xi\|^2),
\]
where $(\mathrm{d}\exp_X)^{-1}(\xi) = \xi + \frac{1}{2}[X, \xi] 
+ \frac{1}{12}[X,[X,\xi]] + \cdots$ is the full BCH series in~$\xi$
(convergent for $\|X\| < 2\pi$).
Substituting this back, the adapter evaluates to:
\[
  \bigl(X + (\mathrm{d}\exp_X)^{-1}(\xi) + O(\|\xi\|^2)\bigr) - X 
  = \xi + \bigl((\mathrm{d}\exp_X)^{-1}(\xi) - \xi\bigr) + O(\|\xi\|^2).
\]
The residual first-order error is 
$\varepsilon_{\mathrm{Lie}} = (\mathrm{d}\exp_X)^{-1}(\xi) - \xi$,
whose leading term is $\frac{1}{2}[X, \xi]$, with higher-order corrections
$\frac{1}{12}[X,[X,\xi]] + \cdots$ that scale as $O(\|X\|^2 \|\xi\|)$.
Because $g_1 \in K$, we have $\|X\| = \|\log(g_1)\| \leq R_K$. 
Since the full series converges and each term is bounded by 
$O(R_K^n \|\xi\|)$, we obtain 
$\|\varepsilon_{\mathrm{Lie}}\| \leq C(R_K) \|\xi\|$
with $C(R_K) = \tfrac{1}{2} C_{\mathrm{Lie}} R_K + O(R_K^2)$,
yielding the result.
\end{proof}

\begin{remark}[Interpretation]~
\begin{itemize}
  \item \textbf{(a)} says that a linear readout~$W$ always exists,
    and its error is controlled by how much the Jacobian $J(g)$ varies
    over the trajectory.
    The \textbf{Jacobian variation is the fundamental obstruction} to
    linear geometric decodability.
  \item \textbf{(b)} says the obstruction is geometric:
    it is the \textbf{curvature} of the feature manifold
    $f(K) \subset \R^d$.
    A flat (affine) feature manifold gives perfect linear readout;
    curvature creates error.
  \item \textbf{(c)} says the nonlinear adapter~$\varphi$
    ``unrolls'' the curved feature manifold, removing the
    Jacobian-variation error from~(a)--(b).
    However, even after perfect unrolling, the Siamese difference
    $\varphi(z_2) - \varphi(z_1)$ incurs a residual from the
    Lie bracket $\tfrac{1}{2}[X, \xi]$, because subtraction in
    $\R^d$ cannot exactly mirror the non-Abelian group operation.
    This residual is $O(\|\xi\| \cdot R_K)$ and is
    an \emph{intrinsic property of $\SE(3)$}, not of the encoder.
    Importantly, the $\se(3)$ bracket satisfies
    $[(\omega_1,v_1),(\omega_2,v_2)]
      = (\omega_1{\times}\omega_2,\;
         \omega_1{\times}v_2 - \omega_2{\times}v_1)$.
    When both $X$ and $\xi$ are pure translations, the bracket
    vanishes entirely.
    When the relative displacement~$\xi$ is a pure translation
    ($\omega_\xi = 0$) but the base pose~$X = (\omega_X, v_X)$ is
    rotated relative to the coordinate origin, a cross-coupling
    term $\omega_X \times v_\xi$ remains; however, the rotational
    component of the bracket still vanishes.
    In general, the residual scales with the rotational extent
    of the region~$K$, and is small when the angular range~$R_K$
    is modest.
\end{itemize}
\end{remark}

\subsubsection{Rotation--Translation Asymmetry}

\begin{theorem}[Rotational Jacobian is Depth-Independent]
\label{thm:rot-trans}
Let $\mathcal{R}(g, s)$ be the rendering of scene~$s$ from
pose $g = (R, t) \in \SE(3)$, and let the encoder~$z$ compute any
weighted spatial average of local image features.
Decompose the Jacobian:
\[
  \mathrm{d}f_g
  = \begin{bmatrix} J_\omega(g) & J_v(g) \end{bmatrix}
  \in \R^{d \times 6},
\]
where $J_\omega = \partial f / \partial \omega$ (rotational) and
$J_v = \partial f / \partial v$ (translational).

\begin{enumerate}
\item[\textbf{(a)}] \textbf{Rotational Jacobian.}
The rotational optical flow at pixel $(u, v)$ is:
\[
  \frac{\partial}{\partial \omega}
  \begin{pmatrix} u \\ v \end{pmatrix}
  =
  \begin{pmatrix}
    -\frac{u v}{f_y} & f_x + \frac{u^2}{f_x}
      & -\frac{f_x v}{f_y} \\[4pt]
    -f_y - \frac{v^2}{f_y}
      & \frac{u v}{f_x} & \frac{f_y u}{f_x}
  \end{pmatrix}.
\]
This depends only on the pixel coordinates and focal lengths
$(f_x, f_y)$---\textbf{not on scene depth~$Z$}.

\item[\textbf{(b)}] \textbf{Translational Jacobian.}
The translational optical flow at pixel $(u, v)$ with depth~$Z$ is:
\[
  \frac{\partial}{\partial v}
  \begin{pmatrix} u \\ v \end{pmatrix}
  = \frac{1}{Z}
  \begin{pmatrix}
    f_x & 0   & -u \\
    0   & f_y & -v
  \end{pmatrix}.
\]
This depends on $Z^{-1}$, which varies across the image and
across scenes.

\item[\textbf{(c)}] \textbf{Consequence for Jacobian variation.}
$J_\omega(g)$ varies only through the encoder's nonlinearity
(how $z$ processes different image content), not through the geometry
of the flow itself.
$J_v(g)$ varies both through the encoder's nonlinearity \emph{and}
through the depth distribution $Z(u,v; g)$.
Since $J_v$ has a strictly larger set of variation sources than
$J_\omega$, the upper bound from Theorem~\ref{thm:global}(b) is
generically larger for translation.
This provides an analytical justification for the empirically observed
inequality:
\[
  \inf_{W_\omega}\sup_{g \in K} \|W_\omega \cdot J_\omega(g) - I_3\|
  \;\leq\;
  \inf_{W_v}\sup_{g \in K} \|W_v \cdot J_v(g) - I_3\|,
\]
with equality only when the scene has constant depth.
Note that this is a heuristic expectation based on the structure
of the upper bounds, not a strict deduction; the actual optimized
readout errors depend on the specific encoder and scene geometry.

\item[\textbf{(d)}] \textbf{Scale ambiguity.}
The translational flow $\partial I / \partial v \propto Z^{-1}$
is invariant under the rescaling
$(v, Z) \mapsto (\lambda v, \lambda Z)$.
Therefore, translation magnitude and depth are confounded:
no encoder operating on pixel intensities alone can disambiguate
$\|v\|$ from~$Z$.
Note that the \emph{local} Jacobian $J_v(g)$ for a fixed scene
generically has full column rank~3, since the depth map $Z(u,v;g)$
is fixed and any non-zero translation induces non-trivial optical flow.
However, scale ambiguity implies that no \emph{universal} linear readout
$W_v$ trained across scenes with varying depth distributions can
reliably recover translation magnitude---only the translation
direction (2~DoF) is universally recoverable without metric priors.
\end{enumerate}
\end{theorem}

\begin{proof}[Proof sketch of~(c)]
The chain rule gives
$J_\omega(g)
 = \frac{\partial z}{\partial I}\big|_{I_g}
   \cdot \frac{\partial I_g}{\partial \omega}$
and
$J_v(g)
 = \frac{\partial z}{\partial I}\big|_{I_g}
   \cdot \frac{\partial I_g}{\partial v}$.
Both share the encoder sensitivity factor $\partial z / \partial I$.
The rotational flow operator $\partial I_g / \partial \omega$ is
depth-independent by~(a); the translational flow operator
$\partial I_g / \partial v$ includes $Z^{-1}(u, v; g)$, which varies
with scene geometry.
This adds an extra source of variation to~$J_v$, so
$\|\nabla J_v\| \geq \|\nabla J_\omega\|$ generically.
\end{proof}

\subsection{Vision Encoder Training, Performance Analysis}
\label{sec:encoder_training}

To understand the necessary and sufficient conditions for spatial
structure to emerge in self-supervised features, we train DINOv2-B
from scratch on a single walking-tour video (Amsterdam WalkingTours,
${\sim}$20 minutes) with systematic ablations of the training recipe.
We evaluate 16 checkpoints spanning component ablations (removing
individual losses or crop strategies), a prototype count sweep
(64 to 65k), and multiple independent reruns of the same
65k-prototype baseline to measure training stochasticity.
All checkpoints are trained for 62{,}500 iterations with effective
batch size~192 ($3 \times$ A100, batch\_size\_per\_gpu\,$=$\,64),
except one checkpoint trained with batch size~128 ($2 \times$ A100).
Each checkpoint is evaluated on ScanNet (3~scenes, ${\sim}$1{,}200 frames)
using two complementary metrics across all 12 transformer layers:
\begin{itemize}
  \item \textbf{Mutual $k$-nn} (\textbf{M1}): a non-parametric measure of
    topological alignment between features and SE(3) camera poses,
    evaluated at frame stride~5 with oracle over layers.
  \item \textbf{Poincar\'e adapter $R^2$}: a parametric measure of
    linear geometric decodability. Following the optimal configuration
    found in our hyperparameter sweep, we use $s\in[10,40]$ range pairing
    with 30{,}000 pairs, a bottleneck dimension of 20, and train for 200 epochs
    averaging over 15 random seeds per checkpoint with oracle over layers.
\end{itemize}

Table~\ref{tab:ablation} reports both metrics. The central finding
is that \emph{topology is stable while geometry is fragile.}

\begin{table}[ht]
\centering
\small
\begin{tabular}{@{}llrrrrrr@{}}
  \toprule
  & & \multicolumn{1}{c}{Topology} & \multicolumn{5}{c}{Geometry (Poincar\'e $R^2$, oracle layer)} \\
  \cmidrule(lr){3-3}\cmidrule(lr){4-8}
  Checkpoint & Ablation & \textbf{M1} & Best L & Overall & Trans & Rot & $\pm\sigma$ \\
  \midrule
  Pretrained DINOv2-B & (142M images)
    & $0.37$ & L7 & $\mathbf{0.61}$ & $\mathbf{0.63}$ & $\mathbf{0.59}$ & $0.04$ \\
  \midrule
  \multicolumn{8}{@{}l}{\textit{Component ablations (all BS\,$=$\,192):}} \\[2pt]
  no\_ibot & $-$\,iBOT
    & $0.39$ & L7 & $0.30$ & $0.30$ & $0.30$ & $0.05$ \\
  no\_koleo & $-$\,KoLeo
    & $\mathbf{0.42}$ & L7 & $0.11$ & $-0.10$ & $0.33$ & $0.11$ \\
  no\_local & 0 local crops
    & $0.40$ & L0 & $0.30$ & $0.26$ & $0.34$ & $0.09$ \\
  no\_global & global scale $[1,1]$
    & $0.35$ & L1 & $0.27$ & $0.12$ & $0.41$ & $0.04$ \\
  no\_dino & $-$\,DINO loss
    & $\mathbf{0.24}$ & L0 & $0.31$ & $0.30$ & $0.32$ & $0.05$ \\
  geo\_regularized & $+$\,geo loss
    & $0.38$ & L10 & $0.26$ & $0.24$ & $0.27$ & $0.06$ \\
  \midrule
  \multicolumn{8}{@{}l}{\textit{Prototype count sweep (all BS\,$=$\,192):}} \\[2pt]
  4k protos & & $0.38$ & L6 & $\mathbf{0.47}$ & $0.42$ & $0.51$ & $0.05$ \\
  12k protos & & $0.39$ & L7 & $0.31$ & $0.37$ & $0.26$ & $0.06$ \\
  1k protos & & $0.39$ & L11 & $0.29$ & $0.34$ & $0.25$ & $0.05$ \\
  256 protos & & $0.38$ & L6 & $0.21$ & $0.17$ & $0.24$ & $0.03$ \\
  64 protos & & $0.38$ & L9 & $0.08$ & $0.00$ & $0.17$ & $0.10$ \\
  \midrule
  \multicolumn{8}{@{}l}{\textit{Reproducibility: identical 65k config, different runs:}} \\[2pt]
  65k\_rerun & & $0.40$ & L6 & $0.35$ & $0.17$ & $0.52$ & $0.04$ \\
  v2\_65k & & $0.38$ & L0 & $0.21$ & $0.11$ & $0.30$ & $0.03$ \\
  65k\_base & & $0.38$ & L3 & $0.15$ & $0.14$ & $0.15$ & $0.06$ \\
  65k\_rerun & & $0.38$ & L8 & $0.08$ & $-0.06$ & $0.23$ & $0.05$ \\
  \midrule
  (BS\,$=$\,128) & 65k, $2{\times}$A100
    & $\mathbf{0.15}$ & L3 & $0.11$ & $-0.04$ & $0.26$ & $0.03$ \\
  random weights & (no training)
    & $0.11$ & L8 & $-0.12$ & $-0.23$ & $-0.00$ & $0.03$ \\
  \bottomrule
\end{tabular}
\caption{\textbf{Training Ablation (ScanNet).}
Topology (mutual $k$-nn metric \textbf{M1}) is measured at stride\,$=$\,5, oracle over 12 layers.
Geometry ($R^2$) is the mean over 15 seeds at the oracle layer.
All checkpoints trained on Amsterdam WalkingTours for 62{,}500 iterations. Topology is stable across training recipes; geometry is fragile and depends critically on data scale.}
\label{tab:ablation}
\end{table}

\paragraph{Topology is stable; geometry is fragile.}
Mutual $k$-nn alignment is remarkably stable: excluding the
\texttt{no\_dino} and (BS\,$=$\,128) outliers, all trained checkpoints
achieve mutual $k$-nn~$\in [0.35, 0.42]$, within 15\% of each other and
comparable to the pretrained model ($0.37$).
In contrast, Poincar\'e adapter $R^2$ varies from $0.08$
to $0.47$ across the same checkpoints---a significantly larger
spread relative to the metric range.
This dissociation reveals that DINOv2 training reliably produces
features that are \emph{topologically} aligned with camera
pose, but the \emph{metric} quality of this alignment
(whether a linear adapter can quantitatively decode pose)
is far more sensitive to training details and stochasticity.

\paragraph{DINO self-distillation shapes depth representation.}
While \texttt{no\_dino} achieves a reasonable overall $R^2$ of $0.31$,
its best layer is \textbf{L0} (the patch embedding), and performance collapses in deeper layers.
This indicates that without the DINO objective, the network fails to build
abstract geometric representations in its deeper layers, defaulting to
low-level pixel correlations available at the input.
Furthermore, \texttt{no\_dino} is the sole ablation that
substantially degrades mutual $k$-nn ($0.24$ vs.\ $0.38$ baseline),
identifying self-distillation as a unique mechanism for topological alignment.

\paragraph{Batch size is critical for manifold formation.}
The (BS\,$=$\,128) checkpoint, trained with effective batch size~128
instead of~192, achieves mutual $k$-nn\,$=$\,$0.15$---near
random weights ($0.11$) and far below all BS\,$=$\,192
checkpoints (${\geq}\,0.35$).
This suggests a sharp phase transition in self-distillation
effectiveness as a function of batch diversity, where insufficient
contrastive signal prevents the formation of a pose-aligned manifold.

\paragraph{Geometry is stochastic even under identical configs.}
Four independent training runs of the same 65k-prototype
configuration yield Poincar\'e $R^2$ ranging from $0.08$
to $0.35$ (Table~\ref{tab:ablation}, reproducibility block),
with the oracle layer migrating from L0 to L8 across runs.
Their mutual $k$-nn scores, however, are nearly identical
($0.38 \pm 0.01$).
This confirms that the topology-to-geometry gap reflects intrinsic
training stochasticity: the feature manifold's global shape
is reproducible, but its local metric structure is not.

\subsection{Full Latent Space Navigation Results}
\label{app:full_navigation}

{Table~\ref{tab:navigation_results_full} provides the complete latent space navigation results including all corrector variants (Linear, MLP-2, MLP-3, Attention) for both DINOv2 and DUSt3R backbones. The condensed version in the main paper (Table~\ref{tab:navigation_results}) reports only the Identity baseline, Inv.\ Poincar\'e, and the Attention corrector.} Table~\ref{tab:navigation_errors_full} reports, for every corrector, the rotation and translation errors of the retrieved frame.
\begin{table}[ht]
\centering
\small
\begin{tabular}{lcccc}
\toprule
 & \multicolumn{2}{c}{\textbf{DINOv2 (Layer 7)}} & \multicolumn{2}{c}{\textbf{DUSt3R (Layer 16)}} \\
\cmidrule(lr){2-3}\cmidrule(lr){4-5}
\textbf{Corrector} & \textbf{Rot.\ err.\ ($^\circ$)} & \textbf{Trans.\ err.\ (m)} & \textbf{Rot.\ err.\ ($^\circ$)} & \textbf{Trans.\ err.\ (m)} \\
\midrule
Identity & 21.2 / 21.6 & 0.32 / 0.37 & 21.2 / 21.6 & 0.32 / 0.37 \\
Inv.\ Poincar\'e & 7.7 / 23.9 & 0.15 / 0.48 & 13.3 / 29.1 & 0.25 / 0.71 \\
Linear & 44.1 / 47.1 & 1.12 / 1.14 & 61.8 / 68.5 & 2.18 / 2.09 \\
MLP-2 & 66.4 / 57.6 & 1.32 / 1.16 & 32.2 / 53.9 & 0.61 / 1.12 \\
MLP-3 & 16.1 / 45.5 & 0.28 / 1.00 & 27.7 / 47.2 & 0.57 / 1.25 \\
Attention & 65.8 / 64.8 & 1.33 / 1.40 & 25.7 / 37.0 & 0.44 / 0.80 \\
\bottomrule
\end{tabular}
\vspace{1mm}
\caption{\textbf{Errors of the retrieved frame} (median / mean over test pairs, averaged over 3 scenes, seed 0): SO(3) geodesic angle and camera-center distance between the retrieved and the target frame.}
\label{tab:navigation_errors_full}
\end{table}

\begin{table}[ht]
\centering
\small
\resizebox{\textwidth}{!}{
\begin{tabular}{lccccccc}
\toprule
\textbf{Corrector} & \textbf{MSE ($\downarrow$)} & \textbf{$R^2$ ($\uparrow$)} & \textbf{Top-1 ($\uparrow$)} & \textbf{Hit@0.1 ($\uparrow$)} & \textbf{Hit@0.2 ($\uparrow$)} & \textbf{Hit@0.3 ($\uparrow$)} & \textbf{Hit@0.5 ($\uparrow$)} \\
\midrule
\multicolumn{8}{c}{\textbf{DINOv2 (Layer 7)}} \\
\midrule
Identity & 0.578 & $-$0.323 & 0.000 & 0.090 & 0.204 & 0.302 & 0.523 \\
Inv.\ Poincar\'e & 0.350 & 0.199 & 0.041 & 0.252 & \textbf{0.479} & \textbf{0.628} & \textbf{0.741} \\
Linear & 0.254 & 0.414 & 0.020 & 0.141 & 0.263 & 0.342 & 0.415 \\
MLP-2 & 0.280 & 0.355 & 0.029 & 0.185 & 0.342 & 0.410 & 0.473 \\
MLP-3 & \textbf{0.219} & \textbf{0.498} & \textbf{0.051} & \textbf{0.261} & 0.451 & 0.532 & 0.581 \\
Attention & 0.258 & 0.407 & 0.018 & 0.167 & 0.323 & 0.383 & 0.422 \\
\midrule
\multicolumn{8}{c}{\textbf{DUSt3R (Layer 16)}} \\
\midrule
Identity & 0.266 & 0.021 & 0.000 & 0.090 & 0.204 & 0.302 & 0.523 \\
Inv.\ Poincar\'e & 0.226 & 0.155 & \textbf{0.030} & \textbf{0.216} & \textbf{0.346} & \textbf{0.467} & \textbf{0.629} \\
Linear & 0.171 & 0.327 & 0.007 & 0.030 & 0.063 & 0.136 & 0.203 \\
MLP-2 & 0.149 & 0.429 & 0.016 & 0.156 & 0.279 & 0.369 & 0.439 \\
MLP-3 & 0.112 & 0.577 & 0.019 & 0.190 & 0.338 & 0.400 & 0.462 \\
Attention & \textbf{0.108} & \textbf{0.597} & 0.024 & 0.201 & 0.357 & 0.455 & 0.535 \\
\bottomrule
\end{tabular}
}
\vspace{1mm}
\caption{{\textbf{Full Latent Space Navigation Results (Averaged over 3 scenes).} We evaluate different methods for predicting the destination feature $\hat{g}_{t+s}$. The Identity baseline (no navigation) calibrates the task difficulty. The zero-shot Inv.\ Poincar\'e approach (using pseudo-inverse $W^\dagger$) performs strongly in topological retrieval, even though parametric correctors like MLP-3 achieve better regression metrics (MSE and $R^2$). The Hit@$\epsilon$ metric measures the fraction of retrievals where the $L_2$ error on the raw 6D twist vector (translation in meters, rotation in axis-angle radians) is below $\epsilon$.}}
\label{tab:navigation_results_full}
\end{table}
\subsection{Nearest-Neighbor Analysis of Latent Navigation}
\label{app:nn_analysis}
In Table~\ref{tab:navigation_results}, the learned correctors reach a lower MSE than Inv.\ Poincar\'e but retrieve worse. One possible explanation is that their predictions fall into sparsely populated regions of the adapted space. To test this, we measure, for every test query, the distance from the predicted feature to its nearest database frame (the retrieved frame), and the distances to the nearest pose-correct and the nearest pose-incorrect frame (pose error below or above $\epsilon = 0.5$). The data do not support the sparse-region explanation. In each of seeds 0--4 and for both backbones, the predictions of the nonlinear correctors (MLP-2, MLP-3, Attention) land \textit{closer} to real frames than those of Inv.\ Poincar\'e (mean nearest-neighbor distance, averaged over seeds 0--4, 1.00--1.07 vs.\ 1.93 for DINOv2, and 1.44--1.60 vs.\ 1.96 for DUSt3R). Relative to the spacing of the database, all predictions are far from any real frame (2.3--4.4 times the median distance between neighboring frames for DINOv2 and 6.7--9.0 times for DUSt3R; seeds 0--2), and the Inv.\ Poincar\'e predictions are the farthest. What does track retrieval is how much closer a prediction is to the nearest pose-correct frame than to the nearest pose-incorrect one: across the five correctors, a lower ratio of these two distances goes with a higher Hit@0.5 (Spearman $\rho$ between $-0.9$ and $-1.0$ in every seed; $-0.86$ for DINOv2 and $-0.92$ for DUSt3R when pooling the seeds). Finally, restricting retrieval to the 6-D image of $W$ lowers the Hit@0.5 of Inv.\ Poincar\'e below that of Identity for both backbones (seeds 0--2). Retrieval in the full adapted space therefore also uses the coordinates outside the row space of $W$, which Inv.\ Poincar\'e carries over from the source frame.
\begin{table}[ht]
\centering
\small
\begin{tabular}{lcccc}
\toprule
 & \multicolumn{2}{c}{\textbf{DINOv2}} & \multicolumn{2}{c}{\textbf{DUSt3R}} \\
\cmidrule(lr){2-3}\cmidrule(lr){4-5}
\textbf{Corrector} & \textbf{NN dist.} & \textbf{Ratio} & \textbf{NN dist.} & \textbf{Ratio} \\
\midrule
Identity & 0.00 & 1.28 & 0.00 & 1.42 \\
Inv.\ Poincar\'e & 1.73 & \textbf{0.87} & 2.11 & \textbf{0.96} \\
Linear & 1.30 & 1.07 & 1.81 & 1.11 \\
MLP-2 & 1.34 & 1.06 & 1.55 & 1.08 \\
MLP-3 & 1.22 & 0.94 & 1.37 & 1.02 \\
Attention & 1.24 & 1.05 & 1.35 & 0.99 \\
\bottomrule
\end{tabular}
\vspace{1mm}
\caption{\textbf{Nearest-neighbor analysis} (seed 0, as in Tables~\ref{tab:navigation_results} and~\ref{tab:navigation_results_full}). NN dist.: mean distance from the predicted feature to the retrieved (nearest) frame in the adapted space. Ratio: distance to the nearest pose-correct frame divided by the distance to the nearest pose-incorrect frame ($\epsilon = 0.5$). Distances are comparable within a backbone only.}
\label{tab:nn_analysis}
\end{table}

\subsection{Interpolation in the Adapted Space}
\label{app:interpolation}
To test whether moving through the adapted space towards the target steadily improves retrieval, we interpolate linearly between the Inv.\ Poincar\'e prediction $\hat{g}_{t+s}$ ($\alpha = 0$) and the true target feature $g_{t+s}$ ($\alpha = 1$), $\hat{g}(\alpha) = (1-\alpha)\,\hat{g}_{t+s} + \alpha\, g_{t+s}$, and repeat the retrieval at each step (same adapters and test pairs as Table~\ref{tab:navigation_results}, seed 0). Table~\ref{tab:interpolation} shows that, on average over test pairs, retrieval improves and the pose error decreases at every step for both backbones; the same holds for the 3-scene averages of each of seeds 0--4. The endpoint $\alpha = 1$ is trivial, because the target frame is in the database. The trend holds on average but not for every pair: on this six-point grid, the pose error is not monotone for 23\% (DINOv2) and 12\% (DUSt3R) of the individual test pairs (mean over seeds 0--2). Appendix~\ref{app:feature_vis} visualizes the raw and adapted feature spaces.
\begin{table}[ht]
\centering
\small
\begin{tabular}{c|cccc|cccc}
\toprule
 & \multicolumn{4}{c|}{\textbf{DINOv2}} & \multicolumn{4}{c}{\textbf{DUSt3R}} \\
$\alpha$ & MSE & Hit@0.5 & Hit@0.1 & Pose err. & MSE & Hit@0.5 & Hit@0.1 & Pose err. \\
\midrule
0.0 & 0.350 & 0.742 & 0.253 & 0.663 & 0.226 & 0.629 & 0.216 & 0.924 \\
0.2 & 0.224 & 0.835 & 0.353 & 0.447 & 0.144 & 0.782 & 0.306 & 0.549 \\
0.4 & 0.126 & 0.873 & 0.490 & 0.346 & 0.081 & 0.916 & 0.463 & 0.199 \\
0.6 & 0.056 & 0.945 & 0.697 & 0.144 & 0.036 & 0.957 & 0.680 & 0.111 \\
0.8 & 0.014 & 1.000 & 0.939 & 0.017 & 0.009 & 0.997 & 0.959 & 0.021 \\
1.0 & 0.000 & 1.000 & 1.000 & 0.000 & 0.000 & 1.000 & 1.000 & 0.000 \\
\bottomrule
\end{tabular}
\vspace{1mm}
\caption{\textbf{Interpolation from the Inv.\ Poincar\'e prediction ($\alpha=0$) to the target feature ($\alpha=1$)} (averaged over 3 scenes, seed 0). Pose err.: mean $L_2$ distance between the 6-D twists of the target and the retrieved relative poses (m and rad), as for Hit@$\epsilon$. The $\alpha = 0$ row is a re-run of the Inv.\ Poincar\'e rows of Table~\ref{tab:navigation_results} (equal up to $0.002$).}
\label{tab:interpolation}
\end{table}
\subsection{Feature-Space Visualizations}
\label{app:feature_vis}
Figure~\ref{fig:feature_vis} visualizes raw DINOv2 features and Poincar\'e-adapted features for one 12-Scenes scene (office2/5b). (a)~In the 2-D PCA projection of the adapted features, the color given by the camera $x$ position changes gradually; in raw DINOv2 it is mixed. (b)~The frames are placed on a regular grid by applying RasterFairy to the 2-D PCA embedding; in the adapted space, frames with similar camera $x$ are grouped together. (c)~Linear interpolation between a source and a target frame in raw feature space (performed in a 16-D PCA projection of the raw features, matching the adapter output dimension) jumps from the source region to the target region, while open-loop navigation in the adapted space follows the ground-truth pose interpolation. The pair in (c) is not random: it is the top-ranked of 300 candidate source/target pairs by raw-path deviation relative to adapted-path deviation, i.e.\ chosen to illustrate the contrast. The adapter used for these figures is trained without a fixed seed, so the adapted panels differ in detail between runs; the patterns described above held in a re-run.
\begin{figure}[ht]
\centering
\begin{tabular}{c}
\includegraphics[width=0.95\linewidth]{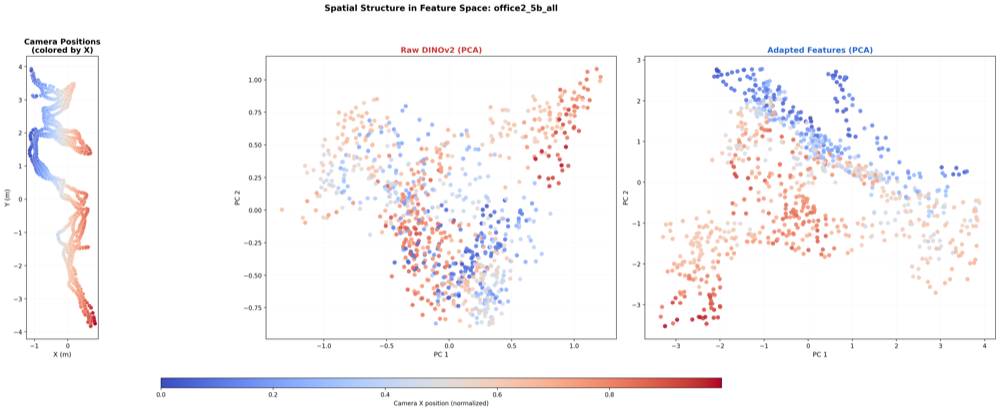} \\ (a) \\[2mm]
\includegraphics[width=0.95\linewidth]{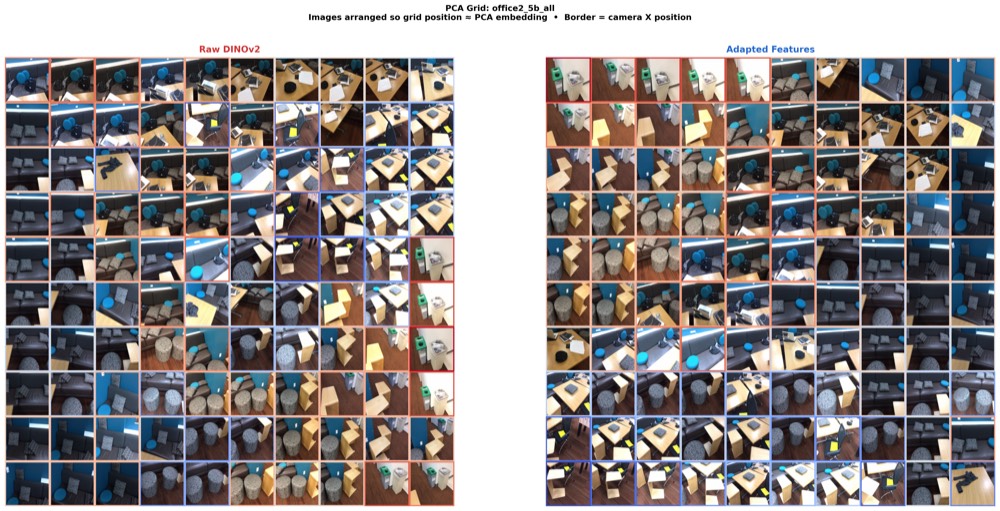} \\ (b) \\[2mm]
\includegraphics[width=0.85\linewidth]{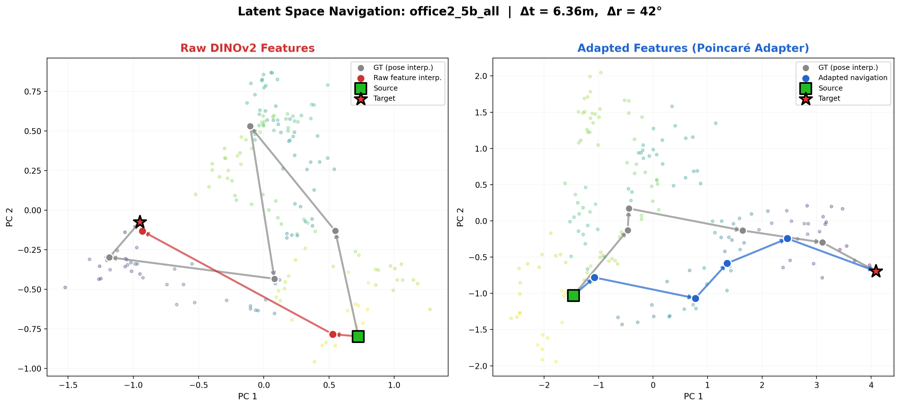} \\ (c)
\end{tabular}
\caption{\textbf{Raw vs.\ adapted feature spaces on 12-Scenes office2/5b.} (a)~2-D PCA of raw DINOv2 and adapted features, colored by camera $x$ position. (b)~Frames placed on a grid by RasterFairy applied to the 2-D PCA embedding (border color: camera $x$). (c)~Raw-feature interpolation (red) vs.\ open-loop navigation in the adapted space (blue), with the ground-truth pose interpolation (gray); the pair is the top-ranked of 300 candidates by raw-path relative to adapted-path deviation, chosen to illustrate the contrast.}
\label{fig:feature_vis}
\end{figure}
\subsection{Cross-Room Navigation Results}
\label{app:cross_room_nav}

{To test the generalizability of the learned geometric structures, we evaluated latent space navigation on a held-out set of 31 test rooms. Poincar\'e adapters were trained on 25 other ScanNet rooms, with the 31 test rooms excluded from training, and we evaluated navigation success using the Hit@0.1 metric. We compared a pure zero-shot inverse (using $W^\dagger$ directly on unseen rooms) against a Ridge-Refit approach (fitting only the linear mapping component of $\Delta P \rightarrow \Delta g$ on the test room, given a frozen adapter). As a control, we repeated both evaluations with randomly initialized, untrained adapters. As shown in Table~\ref{tab:cross_room_nav}, the zero-shot mapping beats the Identity baseline in 58--68\% of rooms, a rate that randomly initialized adapters also reach (42--81\%). The Ridge-Refit correction pushes this to 87--97\% across all three backbone families, versus 6--42\% for randomly initialized adapters. This suggests that the adapter discovers a geometrically structured manifold that carries over to different environments once a linear map is refitted on the new environment.}

\begin{table}[ht]
\centering
\small
{
\begin{tabular}{lcccc}
\toprule
 & \multicolumn{2}{c}{\textbf{Trained adapter}} & \multicolumn{2}{c}{\textbf{Random adapter (3 seeds)}} \\
\cmidrule(lr){2-3} \cmidrule(lr){4-5}
\textbf{Backbone} & \textbf{Zero-Shot} & \textbf{Ridge-Refit} & \textbf{Zero-Shot} & \textbf{Ridge-Refit} \\
\midrule
DINOv2     & 21/31 (68\%) & \textbf{29/31 (94\%)} & 17--20/31 & 2--12/31 \\
DUSt3R     & 20/31 (65\%) & \textbf{30/31 (97\%)} & 13--22/31 & 6--13/31 \\
V-JEPA 2.1 & 18/31 (58\%) & \textbf{27/31 (87\%)} & 14--25/31 & 8--9/31 \\
\bottomrule
\end{tabular}
}
\vspace{1mm}
\caption{{\textbf{Cross-Room Latent Space Navigation.} Number of the 31 held-out test rooms in which navigation (Hit@0.1) beats the no-navigation Identity baseline. The adapters were trained on 25 other ScanNet rooms (seed 0); the random adapters are untrained (range over 3 seeds). Zero-shot uses $W^\dagger$ unchanged; Ridge-Refit refits only the linear map $\Delta P \rightarrow \Delta g$ on the test room, with the adapter frozen. Mean Hit@0.1 is 0.110 for Identity and 0.133--0.142 for Ridge-Refit with the trained adapters (0.112 with random adapters).}}
\label{tab:cross_room_nav}
\end{table}

\subsection{Mixture-of-Experts Poincar{\'e} Adapter}
\label{app:moe_architecture}

To capture the full geometry of complex scenes, we extend the single-chart Poincar{\'e} adapter to a Mixture-of-Experts (MoE) architecture with $K$ local charts. Each expert $k$ possesses its own value projection $W_V^k: \mathbb{R}^d \to \mathbb{R}^g$ and linear readout $W_R^k \in \mathbb{R}^{6 \times g}$. The prediction is a gated combination of expert outputs:

{$$\hat{y} = \sum_{k=1}^{K} \alpha_k(z_t) \cdot W_R^k \cdot (\varphi_k(z_{t+s}) - \varphi_k(z_t))$$}

{Gating is performed via \textit{prototype routing} on the source feature $z_t$ alone:}

{$$\alpha_k(z_t) = \text{softmax}\!\left(\frac{(W_{\text{route}} \, z_t) \cdot p_k}{\exp(\tau)}\right)$$}

{where $W_{\text{route}} \in \mathbb{R}^{r \times d}$ is a shared linear projection, $p_k \in \mathbb{R}^r$ are $K$ learnable prototype vectors, and $\tau$ is a learnable log-temperature scalar. The key design decisions are: (i)~routing depends only on $z_t$ (the observer's current position), not on $z_{t+s}$, enforcing that chart selection depends on \textit{where} the observer is rather than where it is going; (ii)~a shared routing projection reduces parameters by ${\sim}60\times$ compared to per-expert Q/K projections; and (iii)~a learnable temperature allows soft-to-hard expert selection during training, avoiding the ``uniform gating'' trap. For $K{=}1$, the routing layers are omitted entirely, ensuring exact numerical equivalence with the baseline single-chart adapter.}

\subsection{Multi-Step Latent Space Navigation}
\label{app:multi_step_nav}

{We provide qualitative examples of the open-loop Latent Space Navigation described in Section~\ref{sec:navigation_results}. In each case, given a source frame and a target frame, we compute the target pose displacement $\Delta P$ and subdivide it into 4 equal sub-steps. The trajectory is integrated entirely in the frozen latent space by recursively predicting the next feature. We then query the nearest-neighbor image from the test set. Because the navigation operates open-loop, drift can accumulate on long paths. Figure~\ref{fig:multi_step_comb} summarizes three navigation scenarios of increasing difficulty, selected from 40 candidate trajectories by final error, with detailed strip visualizations shown in Figures~\ref{fig:multi_step_t1}, \ref{fig:multi_step_t2}, and~\ref{fig:multi_step_t3}.}

\begin{figure}[ht]
    \centering
    \includegraphics[width=0.9\linewidth]{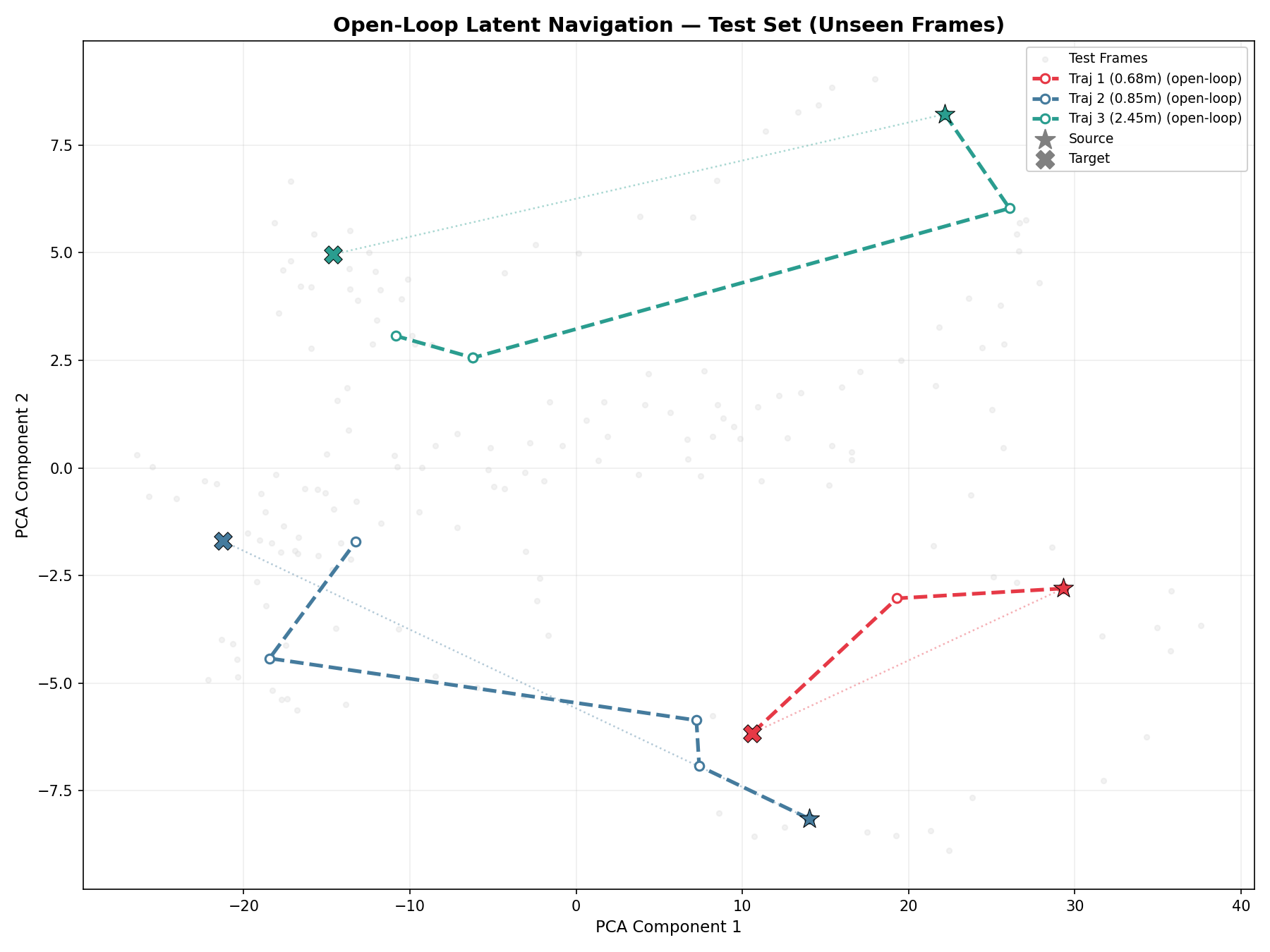}
    \caption{{\textbf{Multi-Step Navigation Trajectories.} Three trajectories of varying difficulty generated via open-loop navigation in the latent space (the trajectories of Figures~\ref{fig:multi_step_t1}--\ref{fig:multi_step_t3}), shown in a 2D PCA projection of the camera poses of the test frames; dashed lines connect the poses of the frames retrieved at each step (stars: sources, crosses: targets). \textbf{Traj~1 ($\Delta=0.68\text{m}, 19^\circ$):} The planner reaches the target with zero final NN retrieval error. \textbf{Traj~2 ($\Delta=0.85\text{m}, 36^\circ$):} Moderate drift, but semantically consistent room traversal. \textbf{Traj~3 ($\Delta=2.45\text{m}, 37^\circ$):} Large translation creates drift, while the rotation error stays below $5^\circ$.}}
    \label{fig:multi_step_comb}
\end{figure}

\begin{figure}[ht]
    \centering
    \includegraphics[width=\linewidth]{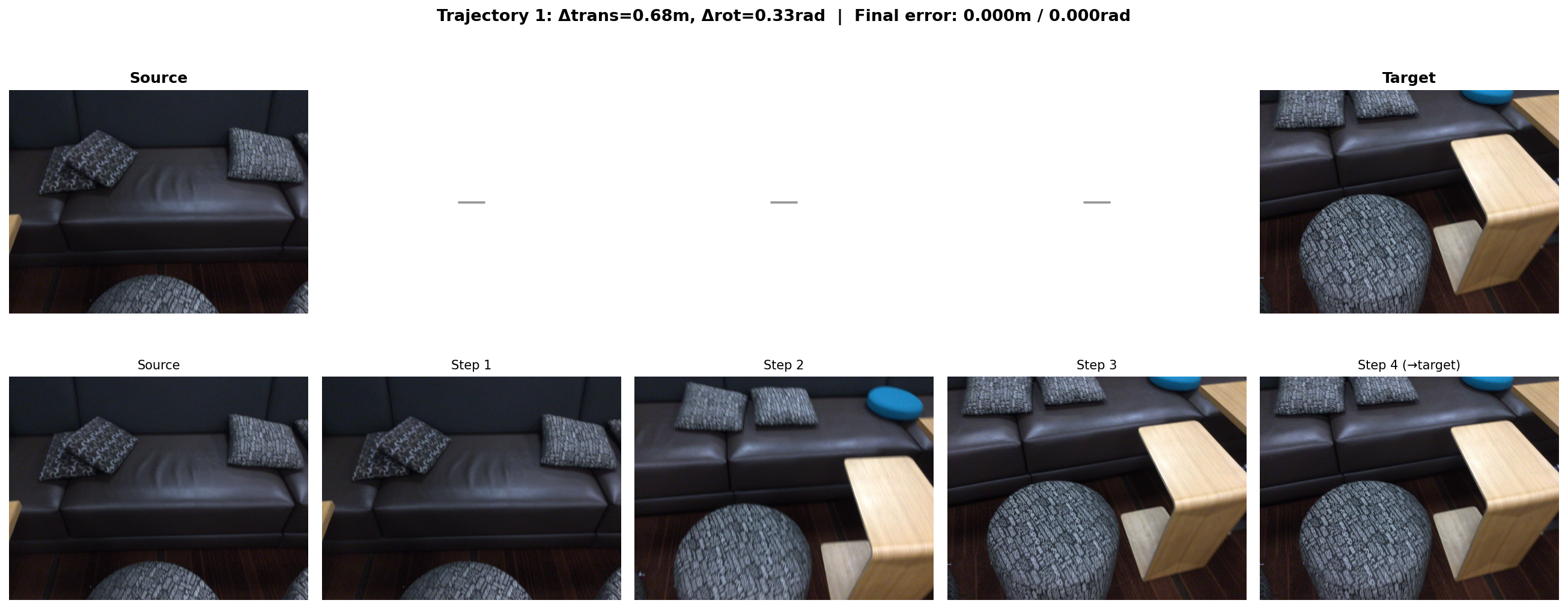}
    \caption{{\textbf{Trajectory 1 (Target Reached).} $\Delta P = 0.68\text{m}, 19^\circ$ rotation. The final retrieved frame is the target frame. (Final error: 0.000m, 0.000rad).}}
    \label{fig:multi_step_t1}
\end{figure}

\begin{figure}[ht]
    \centering
    \includegraphics[width=\linewidth]{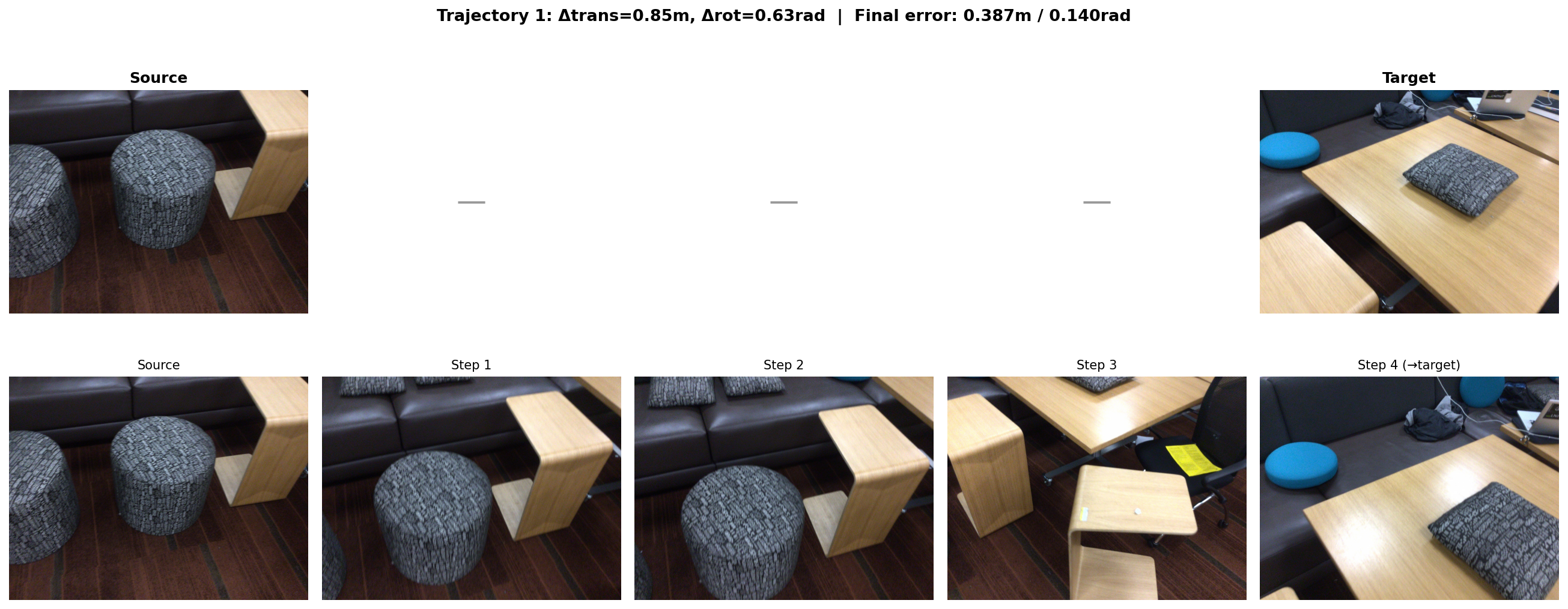}
    \caption{{\textbf{Trajectory 2 (Moderate Drift).} $\Delta P = 0.85\text{m}, 36^\circ$ rotation. The path drifts slightly due to the open-loop integration, but the retrieved frames remain visually coherent and approach the target view. (Final error: 0.387m, 0.140rad).}}
    \label{fig:multi_step_t2}
\end{figure}

\begin{figure}[ht]
    \centering
    \includegraphics[width=\linewidth]{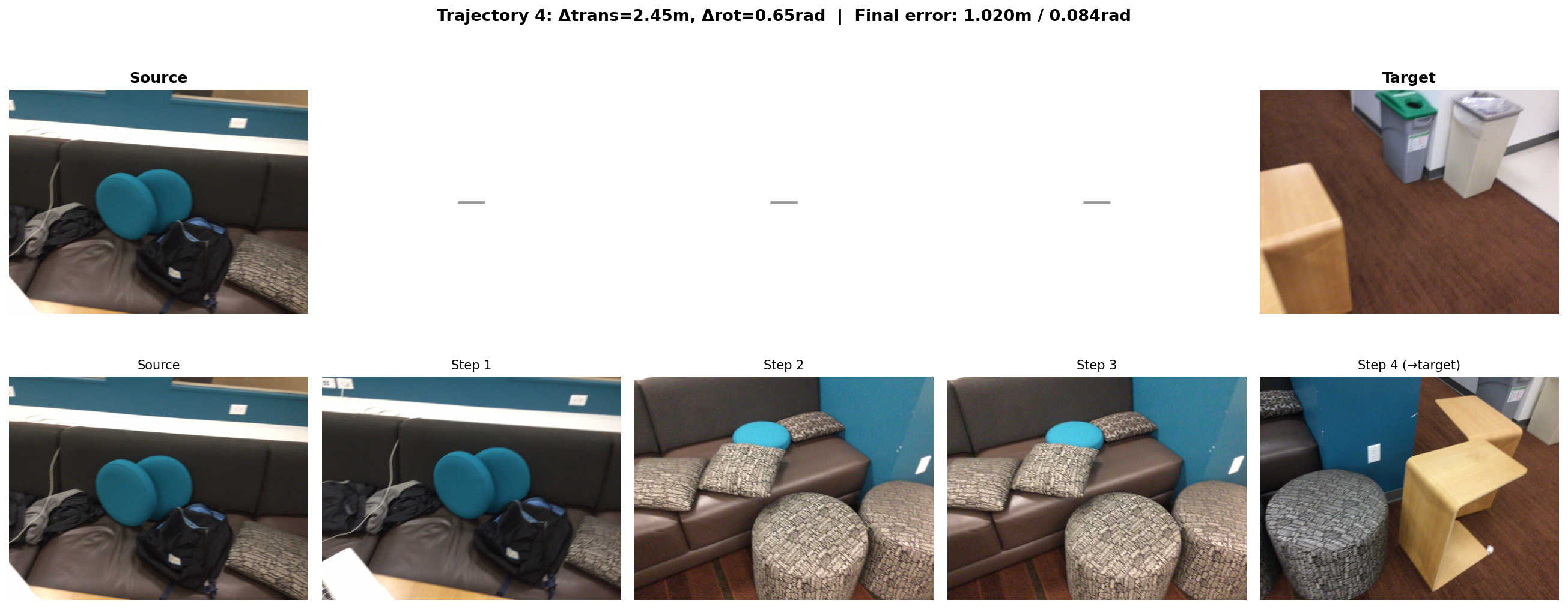}
    \caption{{\textbf{Trajectory 3 (Extreme Displacement).} $\Delta P = 2.45\text{m}, 37^\circ$ rotation. The largest displacement tests the limits of the open-loop linear model. While translation error accumulates, the rotation error stays small ($0.084\text{rad} \approx 5^\circ$). (Final error: 1.020m, 0.084rad).}}
    \label{fig:multi_step_t3}
\end{figure}
\subsection{Preliminary Results on Dynamic Scenes}
\label{app:tum_dynamic}
We ran a preliminary test on three TUM RGB-D sequences~\cite{sturm2012benchmark} in which people move in front of the camera (\texttt{fr3/walking\_halfsphere}, \texttt{fr3/sitting\_halfsphere} and \texttt{fr3/walking\_xyz}). Features are V-JEPA 2.1 (layer 5, mean over tokens), pairs are at most 8 frames apart, and each test sequence is split in time (first 80\% for training, next 10\% for validation, last 10\% for testing). We train the same mixture-of-experts adapter ($K=4$, hidden size 64, bottleneck 32) in three ways: from scratch on the training part of the test sequence; pre-trained on 15 other TUM sequences (10 static, 5 dynamic) and then fine-tuned on the training part of the test sequence; and pre-trained on the 5 dynamic sequences only and then fine-tuned. Table~\ref{tab:tum_dynamic} reports rotation $R^2$. With pre-training on 15 sequences and fine-tuning, rotation $R^2$ reaches $0.68 \pm 0.02$ on \texttt{walking\_halfsphere}, whereas training from scratch gives $0.06 \pm 0.055$ (population s.d.; $\pm 0.067$ sample s.d.). Translation $R^2$ is negative for every sequence, condition and seed. These results are preliminary: four of the five dynamic pre-training sequences were recorded in the same office, with the same people, as the test sequences; the test sets are short (for \texttt{walking\_halfsphere}, the last 107 frames, i.e.\ 856 overlapping pairs), so the s.d.\ reflects only seed variation; and the from-scratch baseline reuses the pre-training schedule without tuning.
\begin{table}[ht]
\centering
\small
\begin{tabular}{lccc}
\toprule
 & & \multicolumn{2}{c}{\textbf{Pre-trained, then fine-tuned}} \\
\cmidrule(lr){3-4}
\textbf{Test sequence} & \textbf{From scratch} & \textbf{15 sequences} & \textbf{5 dynamic sequences} \\
\midrule
\texttt{fr3/walking\_halfsphere} & $0.060 \pm 0.055$ & $\mathbf{0.677 \pm 0.017}$ & $0.652 \pm 0.022$ \\
\texttt{fr3/sitting\_halfsphere} & $-0.742 \pm 0.407$ & $\mathbf{0.537 \pm 0.055}$ & $0.405 \pm 0.148$ \\
\texttt{fr3/walking\_xyz} & $-0.500 \pm 0.230$ & $\mathbf{0.348 \pm 0.015}$ & $0.250 \pm 0.115$ \\
\bottomrule
\end{tabular}
\vspace{1mm}
\caption{\textbf{Rotation $R^2$ on TUM RGB-D dynamic sequences (preliminary).} Mean $\pm$ population s.d.\ over 3 seeds. Translation $R^2$ is negative in every case.}
\label{tab:tum_dynamic}
\end{table}


\clearpage
\section*{NeurIPS Paper Checklist}

\begin{enumerate}

\item {\bf Claims}
\item[] Question: Do the main claims made in the abstract and introduction accurately reflect the paper's contributions and scope?
\item[] Answer: \answerYes{}
\item[] Justification: The contributions listed in Section~\ref{sec:intro} are supported by the results in Sections~\ref{sec:results} and~\ref{sec:analysis}.

\item {\bf Limitations}
\item[] Question: Does the paper discuss the limitations of the work performed by the authors?
\item[] Answer: \answerYes{}
\item[] Justification: Section~\ref{sec:conclusion} discusses the restriction to static scenes; Appendix~\ref{app:tum_dynamic} gives preliminary results on dynamic scenes.

\item {\bf Theory assumptions and proofs}
\item[] Question: For each theoretical result, does the paper provide the full set of assumptions and a complete (and correct) proof?
\item[] Answer: \answerYes{}
\item[] Justification: Appendix~\ref{subsec:theory} states the assumptions of each theorem and gives the proofs.

\item {\bf Experimental result reproducibility}
\item[] Question: Does the paper fully disclose all the information needed to reproduce the main experimental results of the paper to the extent that it affects the main claims and/or conclusions of the paper (regardless of whether the code and data are provided or not)?
\item[] Answer: \answerYes{}
\item[] Justification: The protocol is described in Section~\ref{sec:protocol}, and the adapter and per-experiment settings in Appendix~\ref{app:implementation_details} and Table~\ref{tab:hyperparams}.

\item {\bf Open access to data and code}
\item[] Question: Does the paper provide open access to the data and code, with sufficient instructions to faithfully reproduce the main experimental results, as described in supplemental material?
\item[] Answer: \answerYes{}
\item[] Justification: The code will be publicly released under Apache-2.0, with data preparation instructions; all datasets are public.

\item {\bf Experimental setting/details}
\item[] Question: Does the paper specify all the training and test details (e.g., data splits, hyperparameters, how they were chosen, type of optimizer) necessary to understand the results?
\item[] Answer: \answerYes{}
\item[] Justification: Splits, optimizer, layers, sizes, schedules and seeds are given in Section~\ref{sec:protocol}, Appendix~\ref{app:implementation_details} and Table~\ref{tab:hyperparams}.

\item {\bf Experiment statistical significance}
\item[] Question: Does the paper report error bars suitably and correctly defined or other appropriate information about the statistical significance of the experiments?
\item[] Answer: \answerYes{}
\item[] Justification: Seed standard deviations are reported where several seeds were run (e.g., Table~\ref{tab:alternatives}, Section~\ref{sec:analysis}, Table~\ref{tab:tum_dynamic}); single-seed results are labelled as such.

\item {\bf Experiments compute resources}
\item[] Question: For each experiment, does the paper provide sufficient information on the computer resources (type of compute workers, memory, time of execution) needed to reproduce the experiments?
\item[] Answer: \answerYes{}
\item[] Justification: Encoder training from scratch (Appendix~\ref{sec:encoder_training}, Table~\ref{tab:ablation}) used $2\times$ or $3\times$ NVIDIA A100 GPUs for 62{,}500 iterations per checkpoint. Evaluations and adapter training run on a single NVIDIA A100 or H100 GPU (Appendix~\ref{app:implementation_details}, Table~\ref{tab:hyperparams}); per-experiment runtimes are documented in the release instructions (M1 sweep in Fig.~\ref{fig:stride_sweep}: 2--4 hours; Table~\ref{tab:intrinsic_dim} M2/M3: ${\approx}$1 hour; M4 adapter curves in Fig.~\ref{fig:adapter_b}: estimated 30--60 GPU-hours; Table~\ref{tab:app_dataset_summary} ScanNet sweep: ${\approx}$90 GPU-hours; navigation runs in Tables~\ref{tab:navigation_results}--\ref{tab:navigation_raw}: ${\approx}$3 minutes per run; Fig.~\ref{fig:scaling_curves} multi-room training: ${\approx}$30 minutes on an A100).

\item {\bf Code of ethics}
\item[] Question: Does the research conducted in the paper conform, in every respect, with the NeurIPS Code of Ethics \url{https://neurips.cc/public/EthicsGuidelines}?
\item[] Answer: \answerYes{}
\item[] Justification: The work analyses public vision models on public datasets, used under their terms.

\item {\bf Broader impacts}
\item[] Question: Does the paper discuss both potential positive societal impacts and negative societal impacts of the work performed?
\item[] Answer: \answerNA{}
\item[] Justification: This is foundational analysis of existing visual representations, with no direct societal impact.

\item {\bf Safeguards}
\item[] Question: Does the paper describe safeguards that have been put in place for responsible release of data or models that have a high risk for misuse (e.g., pretrained language models, image generators, or scraped datasets)?
\item[] Answer: \answerNA{}
\item[] Justification: The analysis code and small adapters we will release build on public models and pose no high risk for misuse.

\item {\bf Licenses for existing assets}
\item[] Question: Are the creators or original owners of assets (e.g., code, data, models), used in the paper, properly credited and are the license and terms of use explicitly mentioned and properly respected?
\item[] Answer: \answerYes{}
\item[] Justification: All datasets and models are cited; dataset terms: ScanNet (ToU, non-commercial), 12-Scenes (CC BY-NC-SA 4.0), TUM RGB-D (CC BY 4.0), 7-Scenes (MSR, non-commercial), ARKitScenes (Apple license, non-commercial). Pretrained model licenses: DINOv2 (Apache-2.0), DINOv3 (Meta DINOv3 License), CLIP (MIT), CroCo (CC BY-NC-SA 4.0), DUSt3R (CC BY-NC-SA 4.0), MASt3R (CC BY-NC-SA 4.0), V-JEPA 2.1 (MIT), VGGT (FAIR Noncommercial Research License), MoGe (MIT), Depth Anything v3 (DA3-LARGE, CC BY-NC 4.0), Metric3D v2 (BSD 2-Clause), Perception Encoder (Apache-2.0), VideoMAEv2 (MIT), 4DS and RVM (Apache-2.0 code, CC BY 4.0 models), SD 1.5 (CreativeML OpenRAIL-M), SD 2.1 (CreativeML Open RAIL++-M), LCM (MIT), and SVD (Stability AI Community License).

\item {\bf New assets}
\item[] Question: Are new assets introduced in the paper well documented and is the documentation provided alongside the assets?
\item[] Answer: \answerYes{}
\item[] Justification: The code will be publicly released under Apache-2.0 and documented, including data preparation instructions.

\item {\bf Crowdsourcing and research with human subjects}
\item[] Question: For crowdsourcing experiments and research with human subjects, does the paper include the full text of instructions given to participants and screenshots, if applicable, as well as details about compensation (if any)?
\item[] Answer: \answerNA{}
\item[] Justification: The paper involves no crowdsourcing or human subjects.

\item {\bf Institutional review board (IRB) approvals or equivalent for research with human subjects}
\item[] Question: Does the paper describe potential risks incurred by study participants, whether such risks were disclosed to the subjects, and whether Institutional Review Board (IRB) approvals (or an equivalent approval/review based on the requirements of your country or institution) were obtained?
\item[] Answer: \answerNA{}
\item[] Justification: The paper involves no human subjects.

\item {\bf Declaration of LLM usage}
\item[] Question: Does the paper describe the usage of LLMs if it is an important, original, or non-standard component of the core methods in this research? Note that if the LLM is used only for writing, editing, or formatting purposes and does not impact the core methodology, scientific rigorousness, or originality of the research, declaration is not required.
\item[] Answer: \answerNA{}
\item[] Justification: LLMs are not part of the core methods.

\end{enumerate}

\end{document}